\pdfoutput=1
\documentclass[11pt]{article}

\usepackage[letterpaper,margin=1in]{geometry}
\usepackage{lipsum}

\usepackage[utf8]{inputenc} %
\usepackage[T1]{fontenc}    %

\usepackage{url}            %
\usepackage{booktabs}       %
\usepackage{nicefrac}       %
\usepackage{microtype}      %
\usepackage{enumitem}
\usepackage{dsfont}
\usepackage{multicol}
\usepackage{xcolor}
\usepackage{mdframed}
\usepackage{multirow}

\usepackage{amsthm,amsfonts,amsmath,amssymb,epsfig,color,float,graphicx,verbatim, enumitem}

\usepackage{algorithmicx}
\usepackage[noend]{algpseudocode}
\usepackage{algorithm}

\usepackage{mathtools}
\usepackage{bbm}

\usepackage{thm-restate}

\definecolor{green}{rgb}{0.0, 0.5, 0.0}
\usepackage[colorlinks,citecolor=blue,linkcolor=magenta,bookmarks=true,hypertexnames=false]{hyperref}
\usepackage[nameinlink,capitalize]{cleveref}

\crefformat{equation}{(#2#1#3)}
\crefname{lemma}{lemma}{lemmata}
\crefname{claim}{claim}{claims}
\crefname{theorem}{theorem}{theorems}
\crefname{proposition}{proposition}{propositions}
\crefname{corollary}{corollary}{corollaries}
\crefname{claim}{claim}{claims}
\crefname{remark}{remark}{remarks}
\crefname{definition}{definition}{definitions}
\crefname{problem}{problem}{problems}
\crefname{fact}{fact}{facts}
\crefname{question}{question}{questions}
\crefname{condition}{condition}{conditions}
\crefname{algorithm}{algorithm}{algorithms}
\crefname{assumption}{assumption}{assumptions}
\crefname{notation}{notation}{notation}
\crefname{cond}{Condition}{Conditions}

  {%
   \par\noindent{\bfseries\upshape Proof Sketch\ }%
  }%
  {\qed}

\newtheorem{theorem}{Theorem}[section]
\newtheorem{lemma}[theorem]{Lemma}
\newtheorem{proposition}[theorem]{Proposition}
\newtheorem{corollary}[theorem]{Corollary}
\newtheorem{claim}[theorem]{Claim}

\newtheorem{definition}[theorem]{Definition}

\newtheorem{fact}[theorem]{Fact}

\theoremstyle{definition}

\newtheorem{condition}[theorem]{Condition}

\newtheorem{remark}[theorem]{Remark}

\newcommand{\eps}{\varepsilon}

\newcommand{\dtv}{D_\mathrm{TV}}
\newcommand{\Ind}{\mathds{1}}
\newcommand{\1}{\Ind}

\renewcommand{\Pr}{\operatorname*{\mathbb{P}}}
\newcommand{\Var}{\operatorname*{\mathrm{Var}}}

\newcommand{\E}{\operatorname*{\mathbb{E}}}

\newcommand{\poly}{\operatorname*{\mathrm{poly}}}

\newcommand{\Bernoulli}{\mathrm{Ber}}

\newcommand{\DTV}{D_\mathrm{TV}}

\renewcommand{\vec}[1]{\boldsymbol{\mathbf{#1}}}

\renewcommand{\d}{\mathrm{d}}
\renewcommand{\tilde}{\widetilde}
\renewcommand{\hat}{\widehat}

\def\R{\mathbb R}

\def\N{\mathbb N}

\def\Z{\mathbb Z}

\def\sgn{\mathrm{sgn}}

\newcommand{\cA}{\mathcal{A}}

\newcommand{\cC}{\mathcal{C}}

\newcommand{\cI}{\mathcal{I}}

\newcommand{\cN}{\mathcal{N}}

\newcommand{\cS}{\mathcal{S}}

\newcommand{\cU}{\mathcal{U}}

\newcommand{\cX}{\mathcal{X}}

\newcommand{\op}{\textnormal{op}}
\newcommand{\fr}{\textnormal{F}}

\newcommand{\normal}{\mathcal{N}}
\def\d{\mathrm{d}}
\def\proj{\mathrm{Proj}}

\DeclarePairedDelimiter\abs{\lvert}{\rvert}
\let\vec\mathbf

\newcommand{\multix}{x^{(1)},\ldots,x^{(n)}}

\def\colorful{0}

\ifnum\colorful=1
\newcommand{\inote}[1]{\footnote{{\bf [Ilias: {#1}\bf ] }}}
\newcommand{\dnote}[1]{\footnote{{\bf [Daniel: {#1}\bf ] }}}
\newcommand{\tnote}[1]{\footnote{{\bf [Thanasis: {#1}\bf ] }}}

\else

\newcommand{\inote}[1]{}
\newcommand{\dnote}[1]{}
\newcommand{\tnote}[1]{}

\fi

\allowdisplaybreaks

\title{High-Dimensional Gaussian Mean Estimation under Realizable Contamination}

\author{
Ilias Diakonikolas\thanks{Supported by NSF Medium Award CCF-2107079, ONR award number N00014-25-1-2268, 
and an H.I. Romnes Faculty Fellowship.}\\
University of Wisconsin-Madison\\
{\tt ilias@cs.wisc.edu}\\
\and
Daniel M. Kane\thanks{Supported by NSF Medium Award CCF-2107547.}\\
University of California, San Diego\\
{\tt dakane@cs.ucsd.edu}
\and
Thanasis Pittas\thanks{Supported by NSF Medium Award CCF-2107079.}\\
University of Wisconsin-Madison\\
{\tt pittas@wisc.edu}\\
}
\date{\today}

\begin{document}

\maketitle

\begin{abstract}%
We study mean estimation for a Gaussian distribution with identity covariance in $\R^d$ under a missing data scheme termed \emph{realizable $\eps$-contamination} model. In this model an adversary can choose a function $r(x)$ between 0 and $\eps$ and each sample $x$ goes missing with probability $r(x)$.
Recent work \cite{MVB+2024} proposed this model as an intermediate-strength setting between Missing Completely At Random (MCAR)---where missingness is independent of the data---and Missing Not At Random (MNAR)---where missingness may depend arbitrarily on the sample values and can lead to non-identifiability issues. That work established information-theoretic upper and lower bounds for mean estimation 
in the realizable contamination model. Their proposed estimators 
incur runtime exponential in the dimension, leaving open the possibility of computationally efficient algorithms in high dimensions.
In this work, we establish an information–computation gap 
in the Statistical Query model (and, as a corollary, for Low-Degree Polynomials and PTF tests), showing that algorithms must either use substantially
more samples than information-theoretically necessary or incur exponential runtime. 
We complement our SQ lower bound with an algorithm whose sample–time tradeoff 
nearly matches our lower bound. Together, these results qualitatively characterize the complexity of Gaussian mean estimation under $\eps$-realizable contamination.\looseness=-1
\end{abstract}

\setcounter{page}{0}

\thispagestyle{empty}

\newpage

\section{Introduction}\label{sec:intro}

The fundamental assumption underlying much of classical statistics is that 
datasets consist of i.i.d. samples drawn independently 
from the distribution we aim to learn. In practice, however, this assumption is often violated. A common violation arises when the observations are incomplete. 
This can occur for a variety of reasons, 
for example due to data collection via crowdsourcing \cite{vuurens2011much} or peer grading \cite{PHC+2013,KWL+2013}, 
and has led to the development of a broad literature studying 
estimators that are robust to missing data (see, e.g., \cite{Rub1976a,Tsi2006,little2019statistical} for standard references). 

Different kinds of missingness patterns can be classified based on their nature. The simplest and most benign form of missingness is known as ``Missing Completely at Random'' (MCAR), meaning that the mechanism causing the missingness 
is independent of the data itself. A large body of work studies statistical estimation and inference under MCAR assumptions. Examples include 
sparse linear regression \cite{LW2012,belloni2017linear}, classification \cite{tony2019high,SBC2024}, PCA \cite{elsener2019sparse,zhu2022high,yan2024inference}, covariance and precision matrix estimation \cite{lounici2014high,loh2018high} and changepoint estimation \cite{xie2012change,follain2022high}.

However, the MCAR assumption fails to capture many settings of interest 
where missingness is {\em systematic}. For example, individuals with depression are more likely to submit incomplete questionnaires \cite{carreras2021missing}, and data may be missing for patients who discontinue treatment or go off protocol due to poor tolerability \cite{little2012prevention}. Other commonly studied missingness mechanisms include MAR, in which the missigness dependence on the data values is only through the observed portion of the data \cite{LR2019,seaman2013meant,farewell2022missing}, 
and MNAR, where the missigness may depend in any way on the data \cite{Rob1997,RR1997,scharfstein1999adjusting,shpitser2015missing,ALJY2020,diakonikolas2025linear}. While these models are more expressive, the resulting statistical guarantees are sometimes weak, for example failing to ensure identifiability of the target parameters \cite{MVB+2024}. 

Recent work~\cite{MVB+2024}, proposed and studied 
a different model---which they termed \emph{realizable contamination}; see \Cref{def:cont_model}. This missingness model is not MCAR, 
but it is milder than MAR and MNAR: the missingness depends on the data, but in a more structured %
manner. As we explain below, the realizable contamination model 
can also be viewed as an analogue of Massart noise~\cite{massart2006risk}---a widely studied label-corruption model in supervised learning that lies between purely random and adversarial corruptions---in the context of unsupervised setting.

\begin{definition}[Realizable $\eps$-contamination model]\label{def:cont_model}
    Let $\eps \in (0,1)$ be a contamination parameter.
    Let $P$ be a distribution on a domain $\cX$ with probability density function (pdf) $p:\cX\to \R_+$. An $\eps$-corrupted version $\tilde P$ of $P$ is any distribution that can be obtained as follows: First, an adversary chooses a function $f:\cX\to \R_+$ such that $(1-\eps)p(x) \leq f(x) \leq p(x)$. $\tilde P$ is then defined to be the distribution whose samples are generated as follows:
    \begin{itemize} 
        \item With probability $\int_{\cX} f(x) \d x$ the sample is drawn from the distribution with pdf $f(x)/\int_{\cX} f(x) \d x$.
        \item With probability $1 - \int_{\cX} f(x) \d x$ the sample is set to the special symbol $\perp$.
    \end{itemize}
\end{definition}

\cite{MVB+2024} define this model using a somewhat different formalism that 
is equivalent to \Cref{def:cont_model}. The equivalence of the definitions is discussed in \Cref{app:intro}. Below, we discuss connections between \Cref{def:cont_model} and preexisting concepts in the statistics and ML literature.

First, the realizable contamination model shares some similarities 
with Huber's contamination model from classical robust statistics \cite{Tuk60,Hub64}. In Huber's model, the observed samples are drawn i.i.d.\ 
from a mixture that with probability $(1-\eps)$ outputs a sample from the inlier (clean) distribution and with probability $\eps$ ouputs a sample from an arbitrary outlier distribution. As explained in \Cref{app:intro}, the realizable contamination model of \Cref{def:cont_model} can equivalently be described as a mixture where, with probability $1-\eps$, a sample is drawn from $P$, and otherwise from an MNAR version of $P$. Thus, the inlier component is the same in both models. That said, Huber's model is more powerful due to its ability to introduce arbitrary outlier samples.

Second, the model of \Cref{def:cont_model} can be viewed as an unsupervised analogue of \emph{Massart} noise \cite{massart2006risk}. In supervised learning, Random Classification Noise~\cite{AL88} and adversarial noise represent two extremes, and Massart noise was introduced as a realistic intermediate model. There, the label of $x$ is flipped with probability $\eta(x)$, with 
$\int \eta(x) dx$ equal to the total corruption rate. As shown in \Cref{app:intro}, \Cref{def:cont_model} admits an equivalent description: a sample $x \sim P$ is discarded with probability $1 - f(x)/p(x)$, closely mirroring the Massart noise model.

Lastly, as observed in \cite{MVB+2024}, realizable $\eps$-contamination generalizes several previous approaches for studying restricted forms of MNAR, including certain biased sampling models \cite{vardi1985empirical,gill1988large,bickel1991large,aronow2013interval,sahoo2022learning}, related restrictions known as sensitivity conditions in the causal inference literature \cite{rosenbaum1987sensitivity,zhao2019sensitivity}, and the missingness model used in the truncated statistics literature \cite{DGTZ2018,KTZ2019,DKPZ2024}. The latter can be viewed, at a high level, as corresponding to \Cref{def:cont_model} with $\eps = 1$, albeit with important differences between the two models. A more detailed discussion of related work appears in \Cref{app:intro}.

In this paper, we study the complexity of 
arguably the most fundamental statistical task 
in the presence of realizable $\eps$-contamination: estimating the mean of a multivariate Gaussian with identity covariance. Specifically, given access to 
$\eps$-corrupted samples from $P = \cN(\mu,  I_d)$ on $\R^d$  
where the mean vector $\mu$ 
is unknown, and a desired accuracy $\delta$, 
the goal is to compute an estimate $\hat{\mu}$ that lies within Euclidean distance $\delta$ of the true mean $\mu$.
The only known results for this problem are information-theoretic. \cite{MVB+2024} established matching upper and lower bounds on the task’s sample complexity. For high constant probability of success, it is
\begin{align}\label{eq:sample_compl}
n = \Theta\left(\frac{d}{\delta^2(1-\eps)}\right) + \exp\left( \Theta\left( \frac{\log(1+\tfrac{\eps}{1-\eps})}{\delta} \right)^2 \right).
\end{align}
We note that the first term in \eqref{eq:sample_compl} corresponds to the standard sample complexity for Gaussian mean estimation from the $n(1-\eps)$ clean samples, while the second term captures the additional cost due to realizable $\eps$-contamination. Remarkably, these results hold for all $\eps, \delta \in (0,1)$, highlighting two important differences from Huber’s contamination model (and other robust statistics/missing data models). First, $\delta \to 0$ implies consistent estimation is possible; that is, one can achieve arbitrarily small error, whereas under Huber’s model, estimation is only possible for $\delta \gg \eps$. Second,  
estimation remains possible even for $\eps \ge 1/2$, i.e., 
when the majority of samples are corrupted.

For the sake of completeness, the Appendix provides a self-contained 
and more concise proof of the sample complexity bounds, 
albeit with slightly weaker guarantees: \Cref{app:sample_lb} establishes a lower bound in the one-dimensional case, while \Cref{app:sample_ub} presents 
a (computationally inefficient) algorithm 
using $\tfrac{d}{\eps^2}\exp((\log(1+\eps/(1-\eps))/\delta)^2)$ samples.

Although the information-theoretic aspects of our problem are well-understood, 
much less is known about the computational aspects (the focus of this paper). 
Specifically, essentially the only known multivariate estimator 
is based on computing a cover of the unit sphere with $2^{\Theta(d)}$ directions and applying the one-dimensional estimator to the projections of the samples along each direction. This has runtime $2^{\Theta(d)} \poly(n,d)$ and therefore begs  the following questions:
\begin{center}
\emph{Is there a sample (near-)optimal and polynomial-time mean estimator 
in the presence of realizable contamination? More broadly, what is the computational complexity of this task?}
\end{center}
In this paper we provide strong formal evidence (in the form of Statistical Query/Low-Degree Polynomial or PTF lower bounds) that the problem exhibits an \emph{information–computation gap}---meaning that either the runtime or sample size required is inherently large. 
Second, we give an algorithm whose sample-time tradeoff nearly matches 
our lower bound. Together, these results qualitatively characterize the complexity of Gaussian mean estimation with realizable contamination.

\subsection{Our Results}

We begin with our first main result: an information-computation gap for Gaussian mean estimation in the $\eps$-realizable contamination model. As is typical for such gaps, rather than proving them unconditionally, one usually establishes them under complexity-theoretic hardness assumptions or for restricted (yet natural) families of algorithms. We state the result for the family of Statistical Query (SQ) algorithms below (and discuss other models later on). Before stating the result, we provide a brief summary of SQ.

\paragraph{SQ Model Basics}
The model, introduced by \cite{Kearns:98} and extensively studied since, 
see, e.g.,  \cite{FGR+13}, considers algorithms that, 
instead of drawing individual samples from the target distribution, 
have indirect access to the distribution using the following oracle:
\begin{definition}[STAT Oracle]
    Let $D$ be a distribution on $\R^d$. A statistical query is a bounded function $f: \R^d \to [-1,1]$. For $\tau > 0$, the $\mathrm{STAT}(\tau)$ oracle responds to the query $f$ with a value $v$ such that $|v - \E_{X \sim D}[f(X)]| \leq \tau$. We call $\tau$ the tolerance of the statistical query.
\end{definition}

\noindent An \emph{SQ lower bound} for a learning problem is an unconditional statement 
that any SQ algorithm for the problem either needs to perform a large number $q$ of 
queries, or at least one query with very small tolerance $\tau$. 
Note that, by Hoeffding-Chernoff bounds, a query of tolerance $\tau$ 
is implementable by non-SQ algorithms by drawing $1/\tau^2$ 
samples and averaging them. Thus, an SQ lower bound intuitively 
serves as a tradeoff between runtime of $\Omega(q)$ and sample complexity of 
$\Omega(1/\tau)$.

\vspace{11pt}
Our SQ lower bound for Gaussian mean estimation under $\eps$-realizable contamination shows that any SQ algorithm must either make 
an exponential number of queries or make at least one query with tolerance $d^{-\tilde{\Omega}\left(\frac{\log(1+\eps/(1-\eps))}{\delta}\right)^2}$ (corresponding to a sample complexity of $d^{\tilde{\Omega}\left(\frac{\log(1+\eps/(1-\eps))}{\delta}\right)^2}$ for sample-based algorithms). The lower bound is established for a testing version of the problem: distinguishing between two means that differ by $\delta$ in $\ell_2$ distance. Since any estimator with accuracy $\delta/2$ solves this testing problem, the SQ hardness directly carries over to the estimation setting.

\begin{restatable}[SQ lower bound]{theorem}{SQLB}\label{thm:SQ}
There exists a sufficiently small absolute constant $c>0$ such that the following holds for all $\eps \in (0,1)$ and $\delta \in (0,c\eps)$.
    Define $m = \lfloor  c \gamma^2 / \log \gamma \rfloor$ for $\gamma:= \frac{1}{\delta} \log\left( 1 + \tfrac{\eps/2}{1-\eps/2} 
        \right)$. For any dimension $d \geq (m \log d)^2$ we have the following:
    Any SQ algorithm that distinguishes between $\cN(0,  I_d)$ and $\cN(\delta   u ,  I_d)$ for $  u$ being a unit vector, under the contamination model of \Cref{def:cont_model} needs either $2^{d^{\Omega(1)}}$ queries or at least one query with $d^{-m/16}$ tolerance.\looseness=-1
\end{restatable}

\noindent The above result is based on framing our problem as a special case of Non-Gaussian Component Analysis (NGCA), a general testing problem that is known to be hard in many restricted models of computation, beyond the SQ model, including low-degree polynomial methods \cite{BBH+2021}, PTFs \cite{diakonikolas2025ptf} and the Sum-of-Squares framework \cite{diakonikolas2024sum}. As such, we highlight that qualitatively similar hardness results to those in \Cref{thm:SQ} also hold in these models. The formal statements are deferred to \Cref{app:different_models_hardnesss}.

Our second main result presents an algorithm whose leading term in the sample complexity is $d^{O\left(\log(1+\tfrac{\eps}{1-\eps})/\delta\right)^2}$, with a runtime that is nearly polynomial in $n$ and $d$. In light of our SQ lower bound, this conceptually means that (up to certain factors that we explain below) our algorithm achieves the optimal tradeoff between information and computation.

\begin{restatable}[Algorithmic result]{theorem}{EFFICIENTALG}\label{thm:spectral_alg}
    Let $\eps \in (0,1)$ and $\delta \leq \eps$ be parameters.\footnote{If $\delta \gg \eps$, pre-existing algorithms from robust statistics (see, e.g., \cite{DK2023}) obtain error $O(\delta)$.\looseness=-1}
        There exists an algorithm that takes as input samples from an $\eps$-corrupted version of $\cN(\mu,  I_d)$, the parameters $\eps,\delta$ and returns $\hat \mu \in \R^d$ such that $\|\hat{\mu} - \mu\|_2 \leq \delta$ with probability at least $0.9$. Moreover, the sample complexity of the algorithm is $n = \frac{(kd)^{O(k)}}{\eps^2}$ where $k:=\left(\frac{1}{\delta}\log(1+\frac{\eps}{1-\eps})\right)^2$ and the runtime of the algorithm is $\exp(e^{\tilde O(k)}/\eps^2)\poly(n,d)$.
    \end{restatable}

Up to the $k^{O(k)}/\eps^2$ factor (which is independent of the dimension), the sample complexity 
achieved by our algorithm matches our SQ lower bound. 
The runtime of our algorithm is polynomial in $n$ and $d$, 
up to the term $\exp(e^{\tilde O(k)}/\eps^2)$ (which is also dimension-independent).

\subsection{Our Techniques} \label{sec:techniques}

\paragraph{SQ lower bound} Our SQ lower bound builds on the framework of \cite{DKS17}, which establishes the following.
If $A$ is a one-dimensional distribution whose first $m$ moments match those of $\cN(0,1)$, then distinguishing between $\cN(0,I)$ and the $d$-dimensional distribution that agrees with $A$ along an unknown direction and is standard Gaussian in all orthogonal directions requires either $q = 2^{d^{\Omega(1)}}$ queries or tolerance $\tau < d^{-\Omega(m)}$ in the SQ model.
We show that the robust mean estimation problem we consider can be cast in this form. The main challenge is to design an adversary, as in \Cref{def:cont_model}, that corrupts $\cN(\delta,1)$ so that it matches
$m = \tilde \Omega(\log(1+\tfrac{\eps}{1-\eps})/\delta)^2$ moments of $\cN(0,1)$. We accomplish this via a two-step approach.

First, we construct a function $f$ for the adversary in \Cref{def:cont_model} that uses $\eps/2$ of the corruption budget to match the standard Gaussian over a large interval $x \in [-B,B]$, where
$B = \log(1+\tfrac{\eps}{1-\eps})/\delta$. This relies on the observation that the probability density functions of two unit-variance Gaussians whose means differ by $\delta$ are within a $(1 \pm \eps)$ multiplicative factor of each other, except in the $B$-tails.

After this step, the moments do not exactly match those of $\cN(0,1)$ due to discrepancies outside the interval $[-B,B]$. However, for a small constant $c$, the difference in the first $cB^2$ moments is exponentially small because the mass in the tails decays exponentially for Gaussians. Consequently, by adding a very small polynomial $p(x)$ to $f(x)$ on $[-1,1]$, we can eliminate these remaining discrepancies. By imposing the moment-matching constraints and expanding the polynomial $p(x)$ in the basis of Legendre polynomials, we show that such a polynomial indeed $p(\cdot)$ exists.
Adding this polynomial correction corresponds to censoring each point $x$ with probability
$\eps/4 + p(x)(1-\eps/4)\1(|x|<1)/\phi(x)$, where $\phi$ denotes the standard Gaussian pdf.\looseness=-1

\paragraph{Algorithmic result} As discussed in the introduction, from an information-theoretic perspective, a number of samples as shown in \Cref{eq:sample_compl} suffices to estimate the mean up to error $\delta$. The algorithm achieving this guarantee relies on the fact that two Gaussians whose means differ by $\delta$ have probability density functions that are within a multiplicative factor of $1 \pm \eps$ of each other, except in the  $(\log(1+\tfrac{\eps}{1-\eps})/\delta)$-tails. Consequently, by accurately estimating the densities in these tails using $\exp((\log(1+\tfrac{\eps}{1-\eps})/\delta)^2)$ samples, one can rule out any candidate mean $\mu'$ that lies at distance more than $\delta$ from the true mean $\mu$. This yields a one-dimensional estimator; (also see \Cref{app:sample_ub} for more details).
A multivariate extension can be obtained by repeating this procedure along every direction in a fine cover of the unit sphere, thereby learning an $\ell_2$-approximation of $\mu$. However, such a cover must have $2^{\Theta(d)}$ size, rendering the resulting algorithm computationally infeasible.

For computational efficiency, rather than computing tails, we rely on moments as a substitute. Our approach is motivated by the following observation. Let
$k > (\log(1+\frac{\eps}{1-\eps})/\delta)^2$. Then, for two unit-variance Gaussians whose means differ by $\delta$ (e.g., $\cN(0,1)$ and $\cN(\delta,1)$), there must exist a moment of order $t \in [k]$ that differs by at least $\eps$ (\Cref{cor:non-matching}). This structural result motivates the following algorithm: First, we compute a rough estimate $\hat{\mu}$ of the true mean $\mu$ and translate all samples via the transformation $x \gets x - \hat{\mu}$. After this preprocessing step, the samples behave as if they were drawn from a distribution with mean $\mu - \hat{\mu}$, and the remaining goal is to estimate this difference to error $\delta$. For each $t \in [k]$, we compute the order-$t$ moment tensor $T_t$. By the structural result of \Cref{cor:non-matching} we can certify that for all  $v$ such that $\langle v^{\otimes t}, T_t \rangle$ is small below a certain threshold $\eta$ it should hold $|v^\top(\mu-\hat{\mu})| \leq \delta/2$ (\Cref{cl:mean_certification}).
It therefore suffices to restrict attention to the subspace $V$ corresponding to directions with large moments, and to apply the inefficient estimator described above only within this subspace to estimate $\mu-\hat{\mu}$ up to error $\delta/2$. By the triangle inequality, the resulting estimate has error at most $\delta$. The runtime of this approach is $2^{O(\dim(V))}$. Crucially, we show that $\dim(V)$ is small—bounded solely as a function of $\eps$ and $\delta$, with no dependence on $d$. This follows from the fact that the $\ell_2$-norm of the full moment tensor is bounded in terms of $\eps,\delta$, a property that continues to hold under our corruption model (cf.~\Cref{lem:tensor_bound}). A small norm implies that only few eigenvalues can be large, which in turn bounds $\dim(V)$.

There are two technical complications. First, identifying the subspace $V$ is computationally hard: even finding a single direction $v$ with a large projection of a degree-4 moment tensor is intractable. To circumvent this, instead of searching for directions $v$ with small $\langle v^{\otimes t}, T_t \rangle$, we flatten $T_t$ into a matrix $M \in \R^{d} \times \R^{d^{t-1}}$ and define $V$ as the span of singular vectors of $M$ with singular values exceeding $\eta$, which is computable in polynomial time. This relaxation suffices since, for any $v \in V$, $\langle v^{\otimes t}, T_t \rangle \leq  \sup_{ u \in \R^{d^{t-1}}: \| u \|_2=1  } \langle v,   M u \rangle \leq \delta$.

The second challenge concerns the initial rough estimate $\hat{\mu}$. While the robust statistics literature offers many black-box estimators, most assume stronger contamination models and require $\eps < 1/2$, failing when a majority of samples are corrupted. To our knowledge, no polynomial-time rough estimator exists for \Cref{def:cont_model} when $\eps > 1/2$. We therefore use a list-decodable estimator, which tolerates a majority of outliers but, due to non-identifiability, outputs a list of candidate means, one of which is close to $\mu$, but all others could be arbitrarily inaccurate. Fortunately, there is a tournament-based procedure known in the literature that we can use to identify an element of the list that has error comparable to the best among all of them. This enables us to end up with a single warm start vector and completes the proof sketch of the algorithm.

\section{Preliminaries}\label{sec:prelims}

\paragraph{Basic Notation}  We use $\mathbb{Z}_+$ for the set of positive integers. We denote $[n]=\{1,\ldots,n\}$.
For a vector $x$ we denote by $\| x \|_2$ its  Euclidean norm. 
Let $  I_d$  denote the $d\times d$ identity matrix (omitting the subscript when it is clear from the context). 
 We use  $\top$ for the transpose of matrices and vectors.
 For a tensor $T$, we define by $\|T\|_2 = \sqrt{\sum_{i}T_i^2}$ the $\ell_2$ or (Frobenius) norm.
We use $a\lesssim b$ to denote that there exists an absolute universal constant $C>0$ (independent of the variables or parameters on which $a$ and $b$ depend) such that $a\le Cb$. In our notation $a = O(b)$ has the same meaning as $a \lesssim b$ (similarly for $\Omega(\cdot)$ notation) We use $\tilde O$ and $\tilde \Omega$ to hide polylogarithmic factors in the argument.

\subsection{Non-Gaussian Component Analysis (NGCA)} 

We give a brief background on the SQ hardness of the Non-Gaussian Component Analysis problem (NGCA). First, the testing version of NGCA is defined as distinguishing between a standard Gaussian and a Gaussian that has a non-Gaussian component planted in an unknown direction,  defined below.

\begin{definition}[Hidden direction distribution]\label{def:hidden_dir_dist}
    Let $A$ be a distribution on $\R$.
    For a unit vector $v$, we denote by $P_{A,v}$ the distribution with the density $P_{A,v}(x) := A(v^\top x) \phi_{\perp v}(x)$, where $\phi_{\perp v}(x) = \exp\left(-\|x - (v^\top x)v\|_2^2/2\right)/(2\pi)^{(d-1)/2}$, 
i.e., the distribution that coincides with $A$ on the direction $v$ and is standard Gaussian in every orthogonal direction.
\end{definition}

\begin{restatable}[Non-Gaussian Component Analysis (NGCA)]{problem}{NGCA}\label{prob:generic_hypothesis_testing}
    Let $A$ be a distribution on $\R$ and $P_{A,v}$ the distribution from \Cref{def:hidden_dir_dist}.
    We define the following hypothesis testing problem:
\begin{itemize}  
    \item $H_0$: The data distribution is $\cN(0,I_d)$.
    \item $H_1$: The data distribution is $P_{A,v}$, for some vector $v \in \cS^{d-1}$ in the unit sphere.
\end{itemize}
\end{restatable}

A known result is that NGCA is hard in the SQ model if $A$ matches a lot of moments with the standard Gaussian. 
Here we use the statement from Theorem 1.5 in \cite{diakonikolas2024sq} using $\lambda = 1/2$ and $c = (1-\lambda)/8 = 1/16$ and $\nu=0$. Hardness results for other models, beyond SQ are summarized in \Cref{app:different_models_hardnesss}.
\begin{condition}[Moment matching condition]\label{cond:moment_matching}
$\E_{x \sim A}[x^i] - \E_{x \sim \cN(0,1)}[x^i] = 0$ for all $i \in [m]$.
\end{condition}
\begin{proposition}[Theorem 1.5 in \cite{diakonikolas2024sq}]\label{prop:SQhardness_NGCA}
    Let $d,m$ be positive integers with $d \geq (m\log d)^{2}$.
    Any SQ algorithm that solves \Cref{prob:generic_hypothesis_testing} for a distribution $A$ satisfying \Cref{cond:moment_matching}
    requires either $2^{d^{\Omega(1)}}$ many queries or at least one query with accuracy $d^{-m/16}$.
\end{proposition}

\section{Statistical Query Lower Bound: Proof of \Cref{thm:SQ}}\label{sec:SQ}

In order to show \Cref{thm:SQ}, we will prove the following moment-matching proposition.

\begin{proposition}[Moment Matching]\label{prop:moment_matching}
There exists a sufficiently small absolute constant $c>0$ such that the following holds.
    For every $\eps \in (0,1)$ and $\delta \in (0,c\eps)$,
    there exists a distribution $A$ on $\R$ such that  the following statements are satisfied for $m=\lfloor c \gamma^2/\log \gamma\rfloor$ where $\gamma:= \frac{1}{\delta} \log\left( 1 + \tfrac{\eps/2}{1-\eps/2} 
        \right)$:
    \begin{itemize}
        \item $A$ is the conditional distribution on the visible (non-deleted) samples of an $\eps$-corrupted version of $\cN(\delta,1)$ according to \Cref{def:cont_model}.
        \item It holds $\E_{x \sim A}[x^i] = \E_{x \sim \cN(0,1)}[x^i]$ for $i = 1,\ldots, m$.
    \end{itemize}
\end{proposition}

\noindent We briefly explain how \Cref{prop:moment_matching} yields \Cref{thm:SQ}. First note that the conclusion of \Cref{prop:moment_matching} is phrased in terms of the conditional distribution on the visible samples. However, in \Cref{thm:SQ} we are interested in the hypothesis testing between the corrupted versions of $\cN(0,1)$ and $\cN(\mu,\delta)$, i.e., the entire distributions on both visible and non-visible samples. Although this might seem like a discrepancy, it suffices to focus on the conditional distributions on the visible samples since one may also randomly delete an $\eps$-fraction of points from the null hypothesis, so that the number of deleted points matches the corresponding number in the alternative hypothesis. 

For $v$ being the unit vector in the direction of $  \mu$, let $P_{A,  v}$ be the distribution as in \Cref{prob:generic_hypothesis_testing}. Then, $P_{A,  v}$ is an $\eps$-corrupted version of $\cN(\delta   v ,  I)$, and the hypothesis testing problem of \Cref{thm:SQ} is an instance of non-Gaussian component analysis testing with $A$ as the hidden direction distribution. An application of \Cref{prop:SQhardness_NGCA} with  $m=\lfloor c \gamma^2/\log \gamma\rfloor$ then yields \Cref{thm:SQ}. Finally, note that \Cref{thm:SQ} shows hardness of distinguishing between the ground truth $\mu=0$ and $\|\mu\|_2 \geq \delta$. This immediately implies hardness of the problem of estimating $\mu$ up to error $\delta/2$. This is because if one has an estimator it can run it to obtain $\hat \mu$ with $\| \hat \mu - \mu \|_2 \leq \delta/2$ and reject the null hypothesis iff $\|\hat \mu\|_2 > \delta/2$.

\subsection{Proof of \Cref{prop:moment_matching}}

    Denote by $\phi$ the pdf of $\cN(0,1)$. Following the \Cref{def:cont_model} we need to find a function $f:\R \to \R_+$ such that $(1-\eps)\phi(x-\delta) \leq f(x) \leq \phi(x - \delta)$ and the distribution with pdf $f(x)/\int_\R f(x) \d x$ matches the first $m$ moments with $\cN(0,1)$. The argument to do so consists of two parts:
    \begin{enumerate}[leftmargin=*, nosep]
        \item We will show that there exists $g:\R \to \R_+$ such that $(1-\eps)\phi(x-\delta) \leq g(x) \leq \phi(x - \delta)$, $\int_{x \in \R} g(x) \d x = 1-\eps/2$ and $\frac{g(x)}{1-\eps/2}=\phi(x)$ for all $x \in [-B,B]$ for $B:= \tfrac{1}{\delta}\log\left( 1 + \frac{\eps/2}{1-\eps/2} \right)$. This means that $g$ defines an $\eps$-corrupted version of $\cN(\delta,1)$ that matches exactly $\cN(0,1)$ in the entire range $[-B,B]$. Due to the mismatch outside this interval, the first $m$ moments will not match exactly, but will differ by only a small amount.
        \item In order to correct the moments, we will find an appropriate polynomial $p(x)$ to add to the function $g$ from the previous step in $[-1,1]$ so that (i) the distribution with pdf proportional to $f(x) := g(x) + p(x)\1(|x| \leq 1)$ now matches the $m$ first moments with $\cN(0,1)$ exactly and (ii) $f$ satisfies $(1-\eps)\phi(x-\delta) \leq f(x) \leq \phi(x - \delta)$, i.e., it is still a valid $\eps$-corruption of $\cN(\delta,1)$.
    \end{enumerate}

\smallskip

\noindent    We now present the proofs of the two steps in \Cref{lem:match_inside_interval} and \Cref{lem:correct_moments}, respectively.
    \begin{restatable}{lemma}{APPROXMATCH}\label{lem:match_inside_interval}
    Denote by $\phi(x)$ the pdf of $\cN(0,1)$. For any $\eps,\delta \in (0,1)$
        there exists a function $g:\R \to \R_+$ such that $(1-\eps)\phi(x-\delta) \leq g(x) \leq \phi(x - \delta)$, $\int_{x \in \R} g(x) \d x = 1-\eps/2$ and $\frac{g(x)}{1-\eps/2}=\phi(x)$ for all $x \in [-B+\delta/2,B+\delta/2]$ where $B:= \tfrac{1}{\delta}\log\left( 1 + \frac{\eps/2}{1-\eps/2} \right)$.
    \end{restatable}

    \begin{proof}
    For convenience we will prove the claim with everything shifted by 
        $\delta/2$ to the left, i.e., we will show that, 
        \begin{align}\label{eq:sandwich}
            (1-\eps)\phi(x-\delta/2) \leq g(x) \leq \phi(x - \delta/2)
        \end{align}
        as well as $\int_{x \in \R} g(x) \d x = 1-\eps/2$ and $\frac{g(x)}{1-\eps/2}=\phi(x+\delta/2)$ for all $x \in [-B,B]$.
        For simplicity of notation, we will use $p_{+}(x):=\phi(x-\delta/2)$ and $p_{-}(x):=\phi(x+\delta/2)$ to denote the two Gaussian densities for the rest of the proof.

        The main idea is to let $g(x) = (1-\eps/2)p_{-}(x)$ for all $x$ in an interval around zero which is as large as possible without violating the condition \eqref{eq:sandwich}. Once we find which is the biggest possible such interval, we will need to correct $g(x)$ outside of it so that it still respects condition \eqref{eq:sandwich}.

        For the first part of our proof argument (finding the largest interval for which setting $g(x) = (1-\eps/2)p_{-}(x)$ inside it satisfies condition \eqref{eq:sandwich}) we solve the equations  $(1-\eps/2)p_{-}(x)= p_{+}(x)$ and $(1-\eps/2)p_{-}(x)=(1-\eps)p_{+}(x)$. The two solutions are 
        \begin{align*}
            x_{+} = \frac{1}{\delta}\log\left(\frac{1-\eps/2}{1-\eps} \right) \quad \text{and} \quad x_{-}= \frac{1}{\delta}\log\left(1-\frac{\eps}{2}\right).
         \end{align*}
        This means that if we define $B:= \tfrac{1}{\delta}\log\left( 1 + \frac{\eps/2}{1-\eps/2} \right)$ we have that the function defined as $g(x):=p_{-}(x)(1-\eps/2)$ satisfies condition \eqref{eq:sandwich} for all $x \in [-B,B]$.

        We now need to show how to define $g(x)$ outside of $[-B,B]$. We show that it is possible to extend the definition outside of $[-B,B]$ in a way that condition \eqref{eq:sandwich} continues to hold and $\int_{x \in \R} g(x)\d x = 1-\eps/2$. A first, unsuccessful approach would be to set $g(x)=(1-\eps)p_{+}(x)$ for $x > B$ and $g(x)=p_{+}(x)$ for $x < -B$. Although this ensures condition \eqref{eq:sandwich}, the other desideratum $\int_\R g(x) \d x = 1-\eps/2$ is not satisfied. To see this, let us define $A_1$ and $A_2$ be the following areas:
        \begin{align*}
            A_1 &:= \int_{B}^{+\infty} \left((1-\eps)p_{+}(x)-(1-\eps/2)p_{-}(x)\right) \d x, \\
            A_2 &:= \int_{-\infty}^{-B} \left( (1-\eps/2)p_{-}(x)-p_{+}(x) \right) \d x.
        \end{align*}
        Then the integral of $g$ is
        \begin{align} \label{eq:breakdown}
            \int_{x \in \R} g(x)\d x &=  \int_{-\infty}^{+\infty}(1-\eps/2)p_{-}(x)\d x + \int_{x_+}^B ((1-\eps/2)p_(x)-(1-\eps)p_+(x))) \d x +  A_1 - A_2  \\
            & < 1-\frac{\eps}{2} + A_1 - A_2, \notag
        \end{align}
        where the first integral is simply $1-\eps/2$ and the integral from $x_+$ to $B$ is negative (by definition of $x_+$). We can finally check that
        $A_2 > A_1$ to conclude the proof of $\int_{x \in \R} g(x)\d x < 1-\eps/2$. To see this, first note that we can rewrite $A_2 = \int_{-\infty}^{-B} p_{+}(-x) -(1-\eps/2)p_{-}(-x)\d x = \int_{B}^\infty (1-\eps/2)p_{+}(x) - p_{-}(x)\d x$ where we used that $p_{+}(-x)=p_{-}(x)$ and a change of variable. Then,
        \begin{align*}
            A_2 - A_1 = \int_{B}^{\infty}\frac{\eps}{2}\left( p_{+}(x) - p_{-}(x)  \right) \d x > 0 \;.
        \end{align*}

        However, the above choice of $g(x)$ for $x > B$ is not the only one allowed by condition \eqref{eq:sandwich}. We could alternatively choose any $g(x) \in [(1-\eps)p_{+}(x),p_{+}(x)]$ for $x>B$. We just saw that the first extreme choice $g(x) = (1-\eps)p_{+}(x)$ results in $\int_{x \in \R} g(x)\d x < 1-\eps/2$. We will now show that the other extreme choice of setting $g(x)=p_{-}(x)$ in $x>B$ results in $\int_{x \in \R} g(x)\d x > 1-\eps/2$. By continuity this would mean that there exists a way of defining $g(x)$ in $x>B$ that achieves $\int_{x \in \R} g(x)\d x = 1-\eps/2$.

        We now show the remaining claim above, that the choice $g(x) = p_{+}(x)$ for all $x > B$ (and  $g(x)=p_{+}(x)$ for $x < -B$, $g(x)=(1-\eps/2)p_{+}(x)$ for $x \in [-B,B]$ as before) results in $\int_{x \in \R} g(x)\d x > 1-\eps/2$. In this case, similarly to inequality \eqref{eq:breakdown}, we have
        \begin{align*}
            \int_{x \in \R} g(x)\d x =  \int_{-\infty}^{+\infty}(1-\eps/2)p_{-}(x)\d x + \tilde{A}_1 - A_2
        \end{align*}
        where $A_2$ is the same as before, but $\tilde A_1 = \int_{B}^{\infty} \left(p_{+}(x)\d x - (1-\eps/2)p_{-}(x) \right) \d x$. Now, this gives
        \begin{align*}
             \tilde A_1 - A_2 = \int_B^{\infty} \frac{\eps}{2}\left( p_+(x) + p_{-}(x) \right)  \d x > 0 .
        \end{align*}

    \end{proof}

    \begin{restatable}{lemma}{EXACTMATCH}\label{lem:correct_moments}
        Let $\eps \in (0,1),\delta \ll \eps$ and $g(x)$ be as in \Cref{lem:match_inside_interval}. There exists a polynomial $p(x)$ such that the function $f(x):=g(x) + p(x) \1(|x| \leq 1)$ satisfies $(1-\eps)\phi(x-\delta) \leq f(x) \leq \phi(x-\delta)$ and the distribution with pdf $f(x)/\int_\R f(x) \d x$ matches the first $m$ moments with $\cN(0,1)$ for some $m = \Omega(\gamma^2/\log \gamma)$, where $\gamma:= \frac{1}{\delta} \log\left( 1 + \tfrac{\eps/2}{1-\eps/2} 
        \right)$.
    \end{restatable}

     \begin{proof}
        The function $g(x)$ from \Cref{lem:match_inside_interval} satisfies the following:
        \begin{enumerate}
            \item $(1-\eps)\phi(x-\delta) \leq g(x) \leq \phi(x-\delta)$ for all $x \in \R$
            \item $\int_\R g(x) \d x = 1-\eps/2$
            \item $g(x) = \phi(x)(1-\eps/2)$ for $x \in [-B+\delta/2,B+\delta/2]$ where $B:= \tfrac{1}{\delta}\log\left( 1 + \frac{\eps/2}{1-\eps/2} \right)$. 
        \end{enumerate}

        We will show the existence of a polynomial $p$ such that
        \begin{enumerate}
            \item $|p(x)| \leq c\eps$, where $c$ is a sufficiently small absolute constant,
            \item $\int_{-1}^1 p(x) \d x = 0$,
            \item $\int_\R x^i \frac{g(x) + p(x)\1(|x|\leq 1)}{1-\eps/2}\d x = \int_\R x^i \phi(x)\d x$ for all $i \in [m]$.
        \end{enumerate}

        Before showing that such a polynomial exists, we first show how \Cref{lem:correct_moments} follows given the above points. First, the moment matching property in the conclusion of \Cref{lem:correct_moments} directly follows by the third item in the above list. We now show how the part that $(1-\eps)\phi(x-\delta) \leq f(x) \leq \phi(x-\delta)$ for all $x \in \R$ follows from the above. This can be seen by verifying that (i) $\phi(x-\delta)-g(x)=\Omega(\eps)$ for all $x \in [-1,1]$ and (ii) $g(x)-(1-\eps)\phi(x-\delta)=\Omega(\eps)$ for all $x \in [-1,1]$. We show the part (i) since the other part can be seen with identical arguments. The smallest value of $\phi(x-\delta)-g(x)$ happens at $x=-1$. On that point:
        \begin{align*}
            \phi(x-\delta)-(1-\eps/2)\phi(x) &= \phi(x-\delta) - \phi(x) +\frac{\eps}{2}\phi(x)\geq  -O(\delta) + \Omega(\eps) \geq \Omega(\eps),
        \end{align*}
        where the first inequality uses the fact that $\phi(x-\delta)-\phi(x) = x \phi(x)\delta + \frac{\xi^2-1}{2}\phi(\xi)\delta^2$ for some $x - \delta \leq \xi \leq x$, by Taylor's theorem), and we also used that $\phi(x) = \Omega(1)$ for $x \in [-1,1]$. The last inequality above used that $\delta \ll \eps$.

        We now turn to showing the existence of the polynomial $p$. This part of the proof follows an argument similar to the one in \cite{DKS17}. Recall the moment matching condition that we want to ensure: $\int_{-1}^1 x^i p(x) \d x =  \int_{-\infty}^{\infty}x^i\phi(x) \d x - \int_{-\infty}^{\infty} x^i \frac{g(x)}{1-\eps/2} \d x$. Using the fact that $g(x)/(1-\eps/2) = \phi(x)$ in the interval $[-B+\delta,B+\delta]$ the moment matching condition becomes:
        \begin{align}\label{eq:moment_matching}
            \int_{-1}^1 x^i p(x) \d x =  \int_{\R \setminus [-B+\delta,B+\delta]}x^i\phi(x) \d x - \int_{\R \setminus [-B+\delta,B+\delta]} x^i \frac{g(x)}{1-\eps/2} \d x,
        \end{align}
        for $i=1,\ldots,m$.
        First, we express $p(x)$ as a linear combination of Legendre polynomials $P_k$:
        \begin{fact}
            We can write $p(x) = \sum_{k=0}^m a_k P_k(x)$, where $a_k = \frac{2k+1}{2}\int_{-1}^1 P_k(x) p(x) \d x$.
        \end{fact}
        By properties of Legendre polynomials (\Cref{fact:legendre}) we have that $|p(x)| \leq \sum_{k=0}^m |a_k|$ for all $x \in [-1,1]$, thus it suffices to bound the coefficients $a_k$. Towards that end,
        \begin{align}
            \left|  \int_{-1}^1 P_k(x) p(x) \d x  \right| &= \left|  \int_{\R \setminus [-B+\delta,B+\delta]} P_k(x)\phi(x) \d x -  \int_{\R \setminus [-B+\delta,B+\delta]} P_k(x) \frac{g(x)}{1-\eps/2}\d x\right|  \tag{due to \Cref{eq:moment_matching}}\\
            &\leq   \left| \int_{\R \setminus [-B+\delta,B+\delta]} P_k(x)\phi(x) \d x \right| + \left| \int_{\R \setminus [-B+\delta,B+\delta]} P_k(x) \frac{g(x)}{1-\eps/2}\d x \right|
        \end{align}
        We will show how to bound the first term (the proof for the first term is almost identical). First, 
        \begin{align*}
            \left| \int_{\R \setminus [-B+\delta,B+\delta]} P_k(x)\phi(x) \d x \right| 
            &\leq \left| \int_{\R \setminus [-(B-\delta),(B-\delta)]} 4^k|x|^k \phi(x) \d x \right|  \tag{by \Cref{fact:legendre}}\\
            &\leq   \int_{-\infty}^{-(B-\delta)} 4^k|x|^k \phi(x) \d x   +   \int_{B-\delta}^{\infty} 4^k|x|^k \phi(x) \d x  \\
            &= 2 \int_{B-\delta}^{\infty} 4^k x^k \phi(x) \d x \\
            &\lesssim 4^k\int_{\beta}^{\infty} x^k e^{-x^2/2} \d x \tag{denote $\beta:=B-\delta$}\\
            &\leq 4^k\int_{\beta}^{\infty} x^k e^{-\beta^2/2 + \beta} e^{-x} \d x \tag{$-x^2/2+x$ is decreasing for $x > 1$} \\
            &\leq 4^k\int_{0}^{\infty} x^k e^{-\beta^2/2 + \beta} e^{-x} \d x \\
            &\leq 4^ke^{-\beta^2/2+\beta} \int_{0}^{\infty} x^k e^{-x} \d x \\
            &= 4^k e^{-\beta^2/2+\beta} \Gamma(k+1) \tag{$\Gamma(\cdot)$ is the Gamma function}\\
            &\leq 4^k e^{-\beta^2/4} k^k \;.
        \end{align*}
        Combining the inequality $|p(x)| \leq \sum_{k=0}^m |a_k|$ with the formula for $a_k$ and the above bound, we have that
        \begin{align*}
            |p(x)| &\leq \sum_{k=0}^m |a_k| \lesssim \sum_{k=0}^m \frac{2k+1}{2} 4^k e^{-\beta^2/4} k^k \lesssim 4^m e^{-\beta^2/4} m^m \sum_{k=0}^m \frac{2k+1}{2}  \lesssim 4^m e^{-\beta^2/4} m^{m+2} .
        \end{align*}
        If $m$ is a sufficiently small constant multiple of $\frac{\beta^2 - \log(1/\eps)}{\log(\beta^2 - \log(1/\eps))}$, then the RHS above is at most $\eps$. Note that by our assumption $\delta \ll \eps$ we have $\beta := B - \delta = \frac{1}{\delta}\log(1 + \tfrac{\eps/2}{1-\eps/2}) - \delta = \Theta(\frac{1}{\delta}\log(1 + \tfrac{\eps/2}{1-\eps/2}))$ and $\beta - \log(1/\eps) = \Theta\left( \frac{1}{\delta} \log(1+\tfrac{\eps/2}{1-\eps/2}) \right)$.
        Thus, we can further simplify $\frac{\beta^2 - \log(1/\eps)}{\log(\beta^2 - \log(1/\eps))} = \Omega(\gamma^2 \log \gamma)$, where $\gamma:= \frac{1}{\delta} \log\left( 1 + \tfrac{\eps/2}{1-\eps/2} 
        \right)$. Therefore, the moment matching can be achieved for $m$ as high as a sufficiently small constant multiple of $\gamma^2/\log \gamma$.
    \end{proof}

\section{Efficient Mean Estimation with Realizable Contamination: Proof of \Cref{thm:spectral_alg}}\label{sec:mainalg}

\subsection{Structural Lemma}\label{sec:structural_lem}

The algorithm will be based on the lemma below establishing that if the ground-truth Gaussian has a mean that deviates from zero, then some sufficiently high-order moment must necessarily deviate from the corresponding standard Gaussian moment. Background on Hermite polynomials can be found in \Cref{app:hermite}.

\begin{proposition}[Structural Result]\label{lem:non_match}
        Let $\eps \in (0,1)$ and $\delta \in (0, \log(1+\frac{2\eps}{1-\eps}))$ be parameters. Let $k$ be an even integer which satisfies $k \geq 3 \left(\frac{1}{\delta}\log(1+\frac{2\eps}{1-\eps})\right)^2$ and $P=\cN(\delta,1)$ be a Gaussian distribution. Then for any $\eps$-corrupted version $\tilde{P}$ of $P$ under the model of \Cref{def:cont_model}, if $P'$ denotes the conditional distribution of $\tilde P$ on the non-missing samples, it holds $\E_{x \sim P'}[x^k] - \E_{z \sim   \cN(0,1)}[z^k]  > \eps$.
    \end{proposition}
    \begin{proof}
        We will prove this by showing the following two claims:
        \begin{enumerate}
            \item $\E_{y \sim P}[y^k] > (1+2\eps/(1-\eps)) \E_{z \sim   \cN(0,1)}[z^k]$.
            \item $\E_{x \sim P'}[x^k] \geq (1-\eps) \E_{y \sim P}[y^k]$.
        \end{enumerate}

        These two claims suffice, because if we combine them, we obtain
        \begin{align*}
            \E_{x \sim P'}[x^k] \geq (1-\eps) \E_{y \sim P}[y^k]  \geq  (1-\eps) (1+2\eps/(1-\eps)) \E_{z \sim \cN(0,1)}[z^k] \geq (1 + \eps) \E_{z \sim \cN(0,1)}[z^k] \;.
        \end{align*}
        Rearranging, this means that $\E_{x \sim P'}[x^k] - \E_{z \sim \cN(0,1)}[z^k] \geq \eps \E_{z \sim \cN(0,1)}[z^k] \geq \eps$. We now show the two claims. The second claim follows directly by the definition of the contamination model: if $f$ denotes the function used in \Cref{def:cont_model}, then $f(x) \geq (1-\eps) p(x)$ (where $p$ is the pdf of $P$). Thus\looseness=-1
        \begin{align*}
            \E_{x \sim P'}[x^k] = \int_{\R}x^k \frac{f(x)}{\int_\R f(z) \d z} \d x 
            \geq (1-\eps) \int_{\R}x^k \frac{p(x)}{\int_\R f(z) \d z} \d x 
            \geq (1-\eps) \int_{\R}x^k p(x) 
            = (1-\eps) \E_{y \sim P}[y^k].
        \end{align*}

        We now move to the first claim, i.e., that $\E_{z \sim \cN(0,1)}[(z+\delta)^k] \geq \E_{z \sim \cN(0,1)}[z^k] (1+2\eps/(1-\eps))$ when $k \geq 3\left(\frac{1}{\delta}\log(1+\frac{2\eps}{1-\eps})\right)^2$. First, we recall that $\E_{z \sim \cN(0,1)}[z^k] = (k-1)!!$. Using the binomial theorem, we can write the other non-centered moment as follows:
        \begin{align}
        \E_{z \sim \cN(0,1)}[(z+\delta)^k]
        &= \sum_{\substack{j=0 \\ j\ \text{even}}}^{k}
            \binom{k}{j}\,\delta^{k}\,
            \E_{z \sim \cN(0,1)}[z^{k-j}] 
        = \sum_{\substack{j=0 \\ j\ \text{even}}}^{k}
            \binom{k}{j}\,\delta^{k}\,(k-j-1)!! \;. \label{eq:sum}
        \end{align}
        We will now rewrite the right hand side above. First, for the $(k-j-1)!!$ we have the following. We can rewrite this as follows by taking the first $j/2$ odd factors of $(k-1)!!$:
        \begin{align}
            (k-j-1)!! = \frac{(k-1)!!}{(k-1)(k-3)\cdots(k-j+1)} \geq \frac{(k-1)!!}{k^{j/2}}, \label{eq:double_fact}
        \end{align}
        where we used that every factor in the denominator is at most $k$.
        We also have the following for the binomial coefficient:
        \begin{align}
            \binom{k}{j}  = \frac{k(k-1)\cdots(k-(j-1))}{j!} 
            = \frac{k^j}{j!}\prod_{i=0}^{j-1}\left(1 - \frac{i}{k} \right)
            \geq \frac{k^j}{j!}\prod_{i=0}^{j-1}\left(1 - \frac{i}{j} \right)  \frac{k^j}{j!} \frac{j!}{j^j} = \frac{k^j}{j^j} .\label{eq:binom}
        \end{align}
       Combining Equation \eqref{eq:sum} with inequalities \eqref{eq:double_fact} and \eqref{eq:binom}, we have $\E_{z \sim \cN(0,1)}[(z+\delta)^k]
           \geq (k-1)!! \sum_{\substack{j=0 \\ j\ \text{even}}}^{k}\left( \frac{\sqrt{k} \delta}{j} \right)^j$.
       The sum can be lower bounded by a specific term of the sum. For this we will chose $j$ to be a sufficiently small multiple of $\sqrt{k}\delta$ to obtain
       \begin{align*}
           \E_{z \sim \cN(0,1)}[(z+\delta)^k] \geq (k-1)!! e^{\sqrt{k}\delta/3} = \E_{z \sim \cN(0,1)}[z^k]e^{\sqrt{k}\delta/3} .
       \end{align*}
       Thus, if $k \geq 3 \left(\frac{1}{\delta}\log(1+\frac{2\eps}{1-\eps})\right)^2$ then $\E_{z \sim \cN(0,1)}[(z+\delta)^k]  \geq (1+2\eps/(1-\eps))\E_{z \sim \cN(0,1)}[z^k]$.
       
    \end{proof}

    We will need a Hermite-polynomial version of \Cref{lem:non_match}, obtained by expansion in the Hermite basis and an application of Cauchy–Schwarz.

      \begin{restatable}{corollary}{MOMENTMISMATCH}\label{cor:non-matching}
        Let $\eps \in (0,1)$ and $\delta \in (0, \log(1+\frac{2\eps}{1-\eps}))$ be parameters. Let $k$ be an even integer which satisfies $k \geq 3 (\frac{1}{\delta}\log(1+\frac{2\eps}{1-\eps}))^2$ and $P=\cN(\delta,1)$ be a Gaussian distribution. Then for any $\eps$-corrupted version $\tilde{P}$ of $P$ under the model of \Cref{def:cont_model}, if $P'$ denotes the conditional distribution of $\tilde P$ on the non-missing samples, it holds  $\left| \E_{x \sim P'}[h_{k}(x)]  \right| > \frac{\epsilon}{(k+1)^{k/2}}$, where $h_k$ is the normalized probabilist's Hermite polynomial.\looseness=-1
    \end{restatable}

    \begin{proof}
        \Cref{lem:non_match} states that $\E_{x \sim P'}[x^k] - \E_{z \sim   \cN(0,1)}[z^k]  > \eps$. We now expand the function $x^k$ in the Hermite basis, i.e., $x^k = \sum_{t=0}^k a_t h_t(x)$ where $a_t := \E_{x \sim \cN(0,1)}[x^k h_t(x)]$. Combining the result from \Cref{lem:non_match} with Cauchy-Schwarz gives the following
        \begin{align*}
        \sqrt{\sum_{t=0}^k a_t^2} &\sqrt{(k+1) \max_{t = 0,\cdot,k} \left|\E_{x \sim P'}[h_t(x)] -  \hspace{-8pt}\E_{z \sim   \cN(0,1)}[h_t(z)] \right|^2}  
         \geq \sqrt{\sum_{t=0}^k a_t^2} \sqrt{\sum_{t=0}^k \left( \E_{x \sim P'}[h_t(x)] -  \hspace{-8pt}\E_{z \sim   \cN(0,1)}[h_t(z)] \right)^2}  \\ 
         &\geq \sum_{t=0}^k a_t\left( \E_{x \sim P'}[h_t(x)] -  \E_{z \sim   \cN(0,1)}[h_t(z)] \right) \ \E_{x \sim P'}[x^k] -  \E_{z \sim   \cN(0,1)}[z^k] > \epsilon.
        \end{align*}
        Rearranging we have
        \begin{align*}
            \max_{t = 0,\cdot,k} \left|\E_{x \sim P'}[h_t(x)] -  \E_{z \sim   \cN(0,1)}[h_t(z)] \right|
            &\geq \frac{\epsilon}{\sqrt{k+1}\sqrt{\sum_{t=0}^k a_t^2}}
            = \frac{\epsilon}{\sqrt{k+1}\sqrt{\E_{z \sim \cN(0,1)}[z^{2k}]}}\\
            &= \frac{\epsilon}{\sqrt{(k+1) \frac{(2k)!}{2^k k!}}}
            \geq \frac{\epsilon}{(k+1)^{k/2}}.
        \end{align*}

        \noindent Finally, noting that $\E_{z \sim   \cN(0,1)}[h_t(z)] = 0$ concludes the proof.
    \end{proof}

\subsection{Algorithm Description}

The pseudocode for the algorithm is given in \Cref{alg:estimation_algorithm_2}.
It begins by invoking a list-decodable mean estimation procedure as a black box, i.e., a procedure that
returns a poly-sized list $L$ of candidate means such that at least one element $\tilde\mu_0$ is $O(\sqrt{\log(1/\eps)})$ from the true mean. As explained in \Cref{sec:techniques} this vector would be a good warm start for the subsequent steps. However, since the identity of this good candidate is unknown we use a standard tournament-style selection procedure from robust statistics \cite{diakonikolas2025entangled} to select a vector $\hat \mu_0$ from the list $L$, with error approximately the best among the errors of $L$'s candidates.

Having this $\hat \mu_0$ warm start vector in hand, the algorithm draws a dataset and shifts all samples by $\hat\mu_0$. After this translation, the data may be viewed as having ground truth mean $\mu - \hat\mu_0$, and the task reduces to estimating this offset. We then have the main, dimension-reduction part of the algorithm. For many values $t$, it computes  the $t$-th order moment tensor and considers the subspace $V_t$ spanned by eigenvectors whose corresponding eigenvalues exceed a carefully chosen threshold. The algorithm then sets $V$ to be the span of all the $V_t$'s. 
By the structural result in \Cref{cor:non-matching}, we show that $\mu - \hat\mu_0$ is largely contained inside $V$. Finally, the algorithm applies a computationally inefficient estimator restricted to $V$ to recover the component of the ground truth mean that lies in this subspace. In the proof of correctness, we will show that the dimension of $V$ is small, thus the complexity of this step is as stated in the theorem.\looseness=-1

\begin{algorithm}[h]
    \caption{Spectral algorithm for multivariate mean estimation}
    \label{alg:estimation_algorithm_2}
    \begin{algorithmic}[1]
        \State \textbf{Input}: Sample access to the distribution of \Cref{def:cont_model}, parameters $\eps,\delta \in (0,1)$.
        \State \textbf{Output}: Vector $\widehat{\mu} \in \R^d$.
        \vspace{6pt}

        \State Let $C$ be a sufficiently large absolute constant,
        $k := \left\lceil C \left(\tfrac{1}{\delta}\log\!\left(1 + \tfrac{\eps}{1-\eps}\right) \right)^2 \right\rceil$
        and $\eta := \frac{1}{C} \frac{\eps}{(k+1)^{k/2}}$.

        \State Use a list-decoding algorithm to compute a list $L$ of candidate mean estimates in $\R^d$ such that $L$ contains at least one $\tilde{\mu}_0$ with
        $\| \tilde{\mu}_0 - \mu \|_2 = O(\sqrt{\log(1/(1-\eps))})$.
        \Comment{Using \Cref{fact:list-decoding}}
        \label{line:list-decoding}

        \State $\hat \mu_0 \gets \textsc{TournamentImprove}(L, \eps, O(\sqrt{\log(1/(1-\eps))}))$.
        \Comment{Using \Cref{fact:pruning}}

        \State Draw samples $S = \{x_1,\ldots,x_n\}$ from \Cref{def:cont_model} for
        $n = (kd)^{Ck}/\eps^2$.
        \label{line:draw_dataset}

        \State Recenter the dataset $S' \gets \{x - \hat{\mu}_0 : x \in S\}$.

        \For{$t = 0,1,\ldots,k$}
            \State Compute the empirical Hermite tensor
            $\widehat{T}_t \gets \frac{1}{n}\sum_{x \in S'} H_t(x)$.
            \Comment{cf.~\Cref{def:Hermite-tensor}}

            \State Let $M(\widehat{T}_t)$ be the flattened $d \times d^{t-1}$ matrix.

            \State Compute the right singular vectors $v_1,\ldots,v_d$ of $M(\widehat{T}_t)$
            with singular values $\sigma_1,\ldots,\sigma_d$.

            \State Define $\cI_t = \{ i \in [d] : \sigma_i > \eta \}$ and
            $V_t = \mathrm{span}(\{v_i : i \in \cI_t\})$.
        \EndFor

        \State Let $V = \mathrm{span}(V_1,\ldots,V_k)$.
        \label{line:subspace}

        \State $\widehat{\mu} \gets \mathrm{BruteForce}(\proj_V(S'), \eps, \delta/2)$.
        \Comment{Using \Cref{thm:brute_force}}
        \label{line:brute_force}

        \State \textbf{return} $\widehat{\mu} + \hat{\mu}_0$.
    \end{algorithmic}
\end{algorithm}

\subsection{Useful Subroutines from Robust Statistics}\label{sec:subroutines}

    We require two subroutines used in \Cref{alg:estimation_algorithm_2}. The first handles mean estimation when more than half the samples are corrupted and, since exact recovery is impossible, outputs a list of candidate means containing one close to the truth.

    \begin{fact}[Mean list-decoding algorithm (see, e.g., \cite{diakonikolas2022list})]\label{fact:list-decoding}
        Let $\eps \in (0,1)$ be a corruption rate parameter and $\tau \in (0,1)$ be a probability of failure parameter.
        There exists an algorithm that uses  $\eps$-corrupted samples from $\cN(\mu,I)$ in the strong contamination model\footnote{Unlike \Cref{def:cont_model}, in the \emph{strong contamination model} $(1-\eps)n$ samples are drawn from $\cN(\mu,I)$ and an adversary can add the remaining $\eps n$ points \emph{arbitrarily}.} and finds a list $L$ of candidate means such that, with probability at least 0.99, there is at least one $\hat{\mu}_0 \in L$ with $\| \hat{\mu}_0 - \mu\|_2 = O(\sqrt{  \log(1/(1-\eps))  })$. The sample complexity of the algorithm is $n=(\log(1/(1-\eps)) d)^{O(\log(1/(1-\eps)))}$, the size of the returned list is $|L| = O(\frac{1}{1-\eps})$ and the runtime of the algorithm is $\poly(n)$.
    \end{fact}

    The second component is a pruning procedure that selects a near-optimal estimate from the list. The procedure runs a one-dimensional robust mean estimator along the line connecting each pair of vectors in the list. For each pair, it disqualifies the element that lies farther from the estimated mean along that line. At the end, any remaining element can be returned.

    The proof is identical to that of \cite{diakonikolas2025entangled}, with the only difference being the choice of the one-dimensional robust mean estimator. Here, we may use an estimator designed for Gaussian mean estimation under an $\eps$-fraction of arbitrary corruptions (unlike \Cref{def:cont_model}, where arbitrary corruptions allow the adversary to modify an $\eps$-fraction of the samples arbitrarily). In particular, the median or trimmed mean suffices, with sample complexity $\log(1/\tau)/\delta^2$ in one dimension. We set $\tau = 1/k^2$ to allow a union bound over all pairs in a list of size $k$.
    
    \begin{fact}[Tournament pruning (see Lemma 4.1 in \cite{diakonikolas2025entangled})]\label{fact:pruning}
    \label{lem:prune}
    Let $C$ be a sufficiently large absolute constant.
    Let $\eps,\delta \in (0,1)$ be parameters.
    Let $L = \{ \mu_1, \cdots, \mu_k \} \subset \R^d$ be a set of candidate estimates of $\mu \in \R^d$. 
    There exists an algorithm \textsc{TournamentImprove} that takes as input
    the list $L$, the parameters $\eps,\delta$ and 
    draws  $n =  O\left( \frac{\log k}{\delta^2} \right)$ samples
    according to the data generation model of \Cref{def:cont_model}
    with mean $\mu \in \R^d$ and corruption rate $\eps$, 
     and outputs some estimate $\mu_{j} \in L$ such that
    $ \| \mu_{j} - \mu \|_2 \leq 
    2 \min_{i \in [k]} \| \mu_i - \mu \|_2 + \delta/2$ with probability at least $0.99$. The runtime of the algorithm is $\poly(n,k, d)$.
    \end{fact}

    \subsection{Proof of Correctness}

In the remainder of this section we  prove \Cref{thm:spectral_alg}, with some details deferred to the Appendices.
    We first argue that the final output has error $\delta$. We then bound the complexity of the algorithm.

    \paragraph{Error analysis} 
    In the second line, we know that the list $L$ contains one estimate $\hat{\mu}_0$ with $\|\hat{\mu}_0-\mu\|_2 = O(\sqrt{\log(1/(1-\eps))})$. In what follows we will analyze the  for loop and show that the estimator produced in its last line has $\ell_2$-error at most $\delta$.

    Consider a $\hat{\mu}_0$ for which $\|\hat{\mu}_0-\mu\|_2 = O(\sqrt{\log(1/(1-\eps))})$. Because of the centering transformation,  we can equivalently analyze things as if there was no re-centering transformation but instead all samples come from the model with mean $\mu$ of bounded norm $\|\mu\|_2 = O(\sqrt{\log(1/(1-\eps))})$.

    First we will require the following concentration lemma, which we show in \Cref{app:hermite}.

    \begin{restatable}[Hermite tensor concentration]{lemma}{HERMITECONCENTRATION}\label{lem:hermite_conc}
        Let $\eta,\eps \in (0,1)$ be parameters, $C$ be a sufficiently large absolute constant and $\mu \in \R^d$ be a vector with $\|\mu\|_2 = O(\sqrt{\log(\frac{1}{1-\eps})})$. Let $\tilde P$ be an $\eps$-corrupted version of $\cN(\mu,I)$ (cf. \Cref{def:cont_model}) and let $P'$ denote the conditional distribution of $\tilde P$ on the non-missing samples.  Let $x_1,\ldots,x_n \sim P'$ be i.i.d.\ samples and define $\hat T := \frac{1}{n}\sum_{i=1}^n H_k(x_i)$, and $T:=\E_{x \sim P'}[H_k(x)]$, where $H_k(x)$ denotes the Hermite tensor from \Cref{def:Hermite-tensor}. If $n > C \frac{d^{3k}2^{O(k)}(k \log(\tfrac{1}{1-\eps}))^{k/2}}{(1-\eps)\eta^2 \tau}$, then with probability at least $1-\tau$ we have that $\left\| \hat{T} - T  \right\|_2 \leq \eta$.\looseness=-1
    \end{restatable}

    We will use $\eta := \frac{1}{C} \frac{\eps}{(k+1)^{k/2}}$ as in the pseudocode. The number of samples in \Cref{thm:spectral_alg} has been chosen to allow a union bound over $t=0,\ldots,k$ and all iterations of the for loop of the algorithm so that we have
    \begin{align}\label{eq:good_event}
        \left\| \hat{T}_t - T_t  \right\|_2 \leq \eta  \quad \forall t=0,\ldots,k
    \end{align}
    with probability at least $0.99$. Using the formula for the sample complexity of \Cref{lem:hermite_conc} and simplifying it bit, it can be seen that this concentration event can be achieved with ${(kd)^{O(k)}}{\eps^{-2}}$ samples.\looseness=-1

    We now focus on showing that the estimate $\widehat{\mu}$ towards the end of the for loop has error $\delta$. We will first argue that the mean has a small component of size at most $\delta/2$ in the subspace $V^{\perp}$. Since we estimate the mean up to error $\delta/2$ on the subspace $V$, this will immediately mean that the total error is at most $\delta$.

    \begin{claim}[Mean certification]\label{cl:mean_certification}
        Consider the notation of \Cref{alg:estimation_algorithm_2} and that the event of \eqref{eq:good_event} holds. The subspace $V$ mentioned in \Cref{alg:estimation_algorithm_2} satisfies $v \in V^{\perp} \Rightarrow |v^\top \mu| \leq \delta/2$.
    \end{claim}
    \begin{proof}
    Let $\tilde P$ denote the $\eps$-corrupted version of $\cN(\mu,I)$ according to \Cref{def:cont_model} and let $P'$ denote the conditional distribution on the non-missing samples.
        We prove the claim by contradiction. Suppose that $|v^\top \mu| > \delta/2$. Then by \Cref{cor:non-matching} there exists a $t \in \{0,\ldots,k\}$ with $| \E_{x \sim P'}[h_t(v^{\top}x)] | > \eta$.
        Using the Hermite tensor property (see \Cref{app:hermite} for background on Hermite tensors) $\E_{x \sim P'}[h_t(v^{\top}x)] = \langle v^{\otimes t},  \E_{x \sim P'}[H_t(x)]\rangle$  this means that $|\langle v^{\otimes t}, T_t  \rangle | > \eta$. 
        By the concentration of the event in \eqref{eq:good_event}, we have that $|\langle v^{\otimes t}, \hat T_t  \rangle | > \eta$.
        Therefore, $v$ belongs to $V_t$, which contradicts our assumption $v \in V^{\perp}$.
    \end{proof}

    \Cref{cl:mean_certification} shows that $ \| \proj_{V^{\perp}}(\mu) \|_2 \leq \delta/2$. For the orthogonal subspace our algorithm computes a $\hat{\mu} \in V$ with  $\|\hat \mu + \hat{\mu}_0 - \proj_{V }(\mu) \|_2 \leq \delta/2$ (because of the guarantee stated in \Cref{thm:brute_force}). Thus by the Pythagorean theorem, our estimator $\widehat{\mu}$ towards the end of the for loop has error at most $\delta$. This concludes the error analysis part of the proof. In the remainder of this section we analyze the complexity of the algorithm.

    \paragraph{Complexity of the list-decoding and tournament subroutines} The sample complexity of this part is given by \Cref{fact:list-decoding} and it can be checked that it is smaller than the $(kd)^{O(k)}$ that is mentioned in the statement of \Cref{thm:spectral_alg}. The runtime of list-decoding is $\poly(n)$. It can also be checked that the sample complexity and runtime of the tournament subroutine step, stated in \Cref{fact:pruning} are smaller than what is stated in \Cref{thm:spectral_alg}.

    \paragraph{Complexity of a single iteration of the outer loop} We will bound the runtime of the main for loop, conditioned on the event that $\hat{\mu}_0$ is within $O(\sqrt{\log(1/(1-\eps))})$ of the true mean. 
    The runtime of every step except the application of the brute force algorithm is polynomial on all parameters $n,d$ and $1/(1-\eps)$, so it remains to analyze the runtime of the brute force algorithm. By \Cref{thm:brute_force}, the runtime is $2^{O(\dim(V))}\poly(n,d)$ thus we need a bound on $\dim(V)$. 
    
    Suppose that we have shown a bound $\| \hat{T}_t \|_2 \leq \gamma$. Having that bound will allow us to argue as follows: If $\sigma_i$ denote the singular values of $M(T_t)$, then $\sqrt{\sum_{i=1}^d \sigma_i^2} = \|M(T_t)\|_\fr = \| T_t \|_2 \leq \gamma$ which means that for the set $\cI_t = \{i: \sigma_i \geq \eta \}$ it holds $|\cI_t| \leq \gamma^2/\eta^2$. This means that $\dim(V_k) \leq \gamma^2/\eta^2$. Thus $\dim(V) \leq \sum_{t=0}^k \dim(V_t) \leq (k+1) \gamma^2/\eta^2$.

    We will show in \Cref{lem:tensor_bound} the $\ell_2$-norm bound $\| \E[\hat{T}_t] \|_2 \leq e^{\tilde{O}(k)}$. Under the event \eqref{eq:good_event} that we conditioned on in the beginning, this will also imply that $\|  \hat{T}_t  \|_2 \leq e^{\tilde{O}(k)}$. Having this and plugging $\gamma = e^{\tilde{O}(k)}$, $\eta = O(\eps/(k+1)^{k/2})$ to the bounds of the previous paragraph we will finally have the bounds below which conclude that the runtime of the algorithm is $\exp(e^{\tilde O(k)}/\eps^2)\poly(n,d)$. 
    \begin{align*}
        \dim(V) \leq (k+1) \frac{\gamma^2}{\eta^2} 
        \leq k \frac{e^{\tilde{O}(k)}}{\eta^2} 
        \leq  \frac{e^{\tilde{O}(k)}}{\eta^2}
        \leq \frac{e^{\tilde{O}(k)}(k+1)^{k}}{\eps^2}
        \leq \frac{k^{\tilde O(k)}}{\eps^2}
        \leq \frac{e^{\tilde{O}(k)}}{\eps^2},
    \end{align*}

    \begin{restatable}[Moment tensor norm bound]{lemma}{MOMENTSIZEBOUND}\label{lem:tensor_bound}
    Let $\tilde P$ be the $\eps$-corrupted version of $\cN(\mu,I)$ mentioned in the statement of \Cref{thm:spectral_alg} and $P'$ denote the conditional distribution on the non-missing samples. Assume that $\|\mu\|_2=O(\sqrt{\log(1/(1-\eps))})$.
        Let $T_t = \E_{x \sim P'}[H_t(x)]$ denote the tensors used in \Cref{alg:estimation_algorithm_2}. We have that  
        \begin{align}\label{eq:complicated_bound}
            \|T_t\|_2  \leq \frac{1}{1-\eps} O\left( \log(1/(1-\eps))  \right)^{t/2} + \frac{1}{1-\eps} \exp\left( O\left( t \log(1/(1-\eps)) \right)  \right).
        \end{align}
    \end{restatable}

    We prove this lemma below. Note that the RHS in inequality \eqref{eq:complicated_bound} can be further bounded from above by the simpler expression $e^{\tilde{O}(k)}$ by using  $t \leq k$, $k = \left( \tfrac{1}{\delta}\log(1+\eps/(1-\eps)) \right)^2$, $\delta \leq \eps$.

    \begin{proof}Using the variational characterization of the $\ell_2$-norm, we have

        \[
        \|T_t\|_2
        = \sup_{\|A\|_2 = 1} \langle A, T_t \rangle
        = \sup_{\|A\|_2 = 1} \, \E_{x \sim P'}[\langle A, H_t(x) \rangle ].
        \]
        Since $H_t(x)$ is a symmetric $t$-tensor and depends only on the symmetrization of $A$, we may restrict without loss of generality to symmetric $A$.  
        For such $A$, the function $p_A(x) := \langle A, H_t(x)\rangle$
        is a degree-$t$ polynomial satisfying $\E_{z \sim \mathcal N(0,I)}[p_A(z)^2] = \|A\|_2^2 = 1$.
        Moreover, every degree-$t$ polynomial that is orthonormal with respect to the Gaussian measure can be written in this way from a unique symmetric tensor $A$.  
        Therefore,
        \begin{align}
        \|T_t\|_2
        = \sup_{\substack{p \text{ degree } t \\ \mathbb{E}_{z\sim\mathcal N(0,I)}[p(z)^2]=1}}
        \E_{x\sim\tilde P}[p(x)].
        \label{eq:var-cahr}
        \end{align}

        Thus we need to show that the $\E_{x \sim P'}[p(x)]$ is bounded for every unit-norm polynomial $p$ of degree $t$. To this end, recall the definition of the distribution $P'$ from \Cref{def:cont_model_alt} 
        \begin{align*}
            \E_{x \sim P'}[p(x)] 
            &=  \E_{x \sim \cN(\mu,I)}[p(x) \,|\, \text{$x$ not missing}]
            = \frac{\E[p(x) \1(\text{$x$ not missing})]}{\Pr_{x \sim \cN(\mu,I)}[ \text{$x$ not missing} ]}\\
            &\leq \frac{1}{1-\eps} \E_{x \sim \cN(\mu,I)}[p(x) \1(\text{$x$ not missing})] \\
            &\leq \frac{1}{1-\eps} \E_{x \sim \cN(\mu,I)}[p(x)] - \frac{1}{1-\eps}\E_{x \sim \cN(\mu,I)}[p(x)\1(\text{$x$ missing})]\\
            &\leq \frac{1}{1-\eps}\left(\left| \E_{x \sim \cN(\mu,I)}[p(x)] \right| + \left|\E_{x \sim \cN(\mu,I)}[p(x)\1(\text{$x$ missing})] \right| \right).
        \end{align*}

        \noindent We will analyze each term separately. For the first term we have the following:
        \begin{align*}
            \E_{x \sim \cN(\mu,I)}[p(x)] &\leq  \left\|\E_{x \sim \cN(\mu,I)}[H_t(x)] \right\|_2 \tag{by the variational characterization of $\ell_2$-norm} \\
            &= \left\| \mu^{\otimes t} \right\|_2/\sqrt{t!} \leq \| \mu \|_2^t/\sqrt{t!} \\
            &= O\left(\log(1/(1-\eps))\right)^{t/2} \tag{by \Cref{fact:exp-h-tensor} and $\| \mu \|_2 = O(\sqrt{\log(1/(1-\eps)})$}
        \end{align*}

        For the second term, we have the following:
        \begin{align*}
            \left|\E_{x \sim \cN(\mu,I)}[p(x)\1(\text{$x$ missing})] \right|
            &\leq \sqrt{\E_{x \sim \cN(\mu,I)}[p^2(x)]} \sqrt{\Pr[\text{$x$ missing}]} \\
            &\leq \sqrt{\E_{x \sim \cN(\mu,I)}[p^2(x)]} \eps \;.
        \end{align*}
        It remains to bound $\E_{x \sim \cN(\mu,I)}[p^2(x)]$. 
        We know that $\E_{x \sim \cN(0,I)}[p^2(x)]=1$, however we need to bound the expectation over a translated Gaussian. By \Cref{fact:multivariate-shift}, we have that $\E_{x \sim \cN(\mu,I)}[p^2(x)] \leq e^{t \|\mu\|^2} \leq e^{O(t \log(1/(1-\eps))})$ (where we used that $\| \mu \|_2 = O(\sqrt{\log(1/(1-\eps)})$). Overall, by putting everything together, we have that 
        \begin{align*}
            \E_{x \sim \bar P}[p(x)]  \leq \frac{1}{1-\eps} O\left( \log(1/(1-\eps))  \right)^{t/2} + \frac{1}{1-\eps} \exp\left( O\left( t \log(1/(1-\eps)) \right)  \right).
        \end{align*}

    \end{proof}

    \section{Conclusions}

    This work studied and essentially characterized the complexity of mean estimation of a spherical Gaussian in the realizable contamination model. Our results suggest several natural questions for future work. A broad direction is to understand the tradeoffs between sample and computational complexity for other statistical tasks in the presence of realizable contamination. A second direction is to investigate broader distributional assumptions. While the Gaussian setting serves as the canonical starting point, a natural next step is to consider subgaussian distributions. However, in full generality, consistency is known to be impossible in this model \cite{MVB+2024}. This raises the question of identifying structured subclasses of subgaussian distributions, or alternative distribution families, for which consistent estimation remains achievable.

    \section*{Acknowledgments}
    We thank Chao Gao for bringing the realizable contamination model to our attention.

\newpage

\bibliographystyle{alpha}
\bibliography{refs}

\newpage 
\appendix
\crefalias{section}{appendix} %

\section*{Appendix}

\section{Omitted Details from \Cref{sec:intro}} 

\subsection{Additional Discussion on Related Work}\label{app:intro}

\paragraph{Comparison of \Cref{def:cont_model} with the definition in \cite{MVB+2024}}
As described in \Cref{sec:intro}, the contamination model (\Cref{def:cont_model}) studied in this paper was proposed in the recent work of \cite{MVB+2024} as an attempt to formalize missingness mechanisms that are non-MCAR yet milder than MAR or MNAR. At first glance, the model appears to be defined slightly differently in eq. (6) of \cite{MVB+2024}. 
The $\eps$-corrupted version $\tilde P$ of the original disrtibution $P$ there is defined as any mixture of the form (as a note regarding notation in that paper, we are using $q=1$ in the notation of that paper, i.e., there is no MCAR component in eq. (6))
\begin{align}\label{eq:model_original}
    (1-\eps) P + \eps Q,
\end{align}
where $Q$ is any MNAR version of $P$ that the adversary can choose.

However, \cite[Proposition 2]{MVB+2024} provides a characterization that establishes its equivalence with the definition used in this paper. As explained following Proposition 2, if samples are interpreted as being generated by first drawing a value and then applying a missingness mechanism, and if $h(x)$ denotes the probability that a sample is missing conditional on the original value being $x$, then Proposition 2 shows that the realizable $\eps$-contamination model is equivalent to the condition $1-\eps \le h(x) \le 1$. In the language of our paper, this characterization leads to the alternative definition of the realizable $\eps$-contamination model stated in \Cref{def:cont_model_alt} below.

\begin{definition}[Contamination model; Alternate definition]\label{def:cont_model_alt}
    Let $P$ be a distribution on a domain $\cX$ with pdf $p:\cX\to \R_+$. An $\eps$-corrupted version $\tilde P$ of $P$ is any distribution that can be obtained as follows: First, an adversary chooses a function $f:\cX\to \R_+$ with $(1-\eps)p(x) \leq f(x) \leq p(x)$. $\tilde P$ is then defined to be the distribution whose samples are generated as follows:
    \begin{enumerate}
        \item Draw $x$ from $P$.
        \item With probability $1 - f(x)/p(x)$, replace $x$ by a special symbol $\perp$.
    \end{enumerate}
    
\end{definition}

\noindent As shown below, it is then straightforward to check that \Cref{def:cont_model_alt} is equivalent to \Cref{def:cont_model}.
\begin{claim}
    \Cref{def:cont_model} and \Cref{def:cont_model_alt} are equivalent.
\end{claim}
\begin{proof}
Let us denote by $X$ the initial value of the sample before the missingness pattern is applied and by $Z$ the sample after its application.
    In \Cref{def:cont_model_alt} we have 
    \begin{align*}
        \Pr[Z \neq \perp] 
        = \int_{\cX}  \Pr[Z \neq \perp | X=x]\Pr[X=x] \d x =\int_{\cX}  \frac{f(x)}{p(x)}p(x)  \d x = \int_{\cX} f(x) \d x,
    \end{align*}
    and we also have
    \begin{align*}
        \Pr[X=x | Z \neq \perp] 
        &= \frac{\Pr[Z \neq \perp | X=x]\Pr[X=x]}{\Pr[Z \neq \perp]} = \frac{\left(1 - \Pr[Z = \perp | X= x] \right) p(x)}{\int_{\cX} f(x) \d x} \\
        &= \frac{\frac{f(x)}{p(x)}p(x)}{\int_{\cX} f(x) \d x} = \frac{f(x)}{\int_{\cX} f(x) \d x} ,
    \end{align*}
    which agrees with \Cref{def:cont_model}.
    
\end{proof}

\paragraph{Robust Statistics}
The field of robust statistics was initiated in the 1960s through the seminal works of Huber and Tukey \cite{Tuk60,Hub64}, with the goal of developing estimators that are robust to data contamination. In this setting, data corruptions are formalized by allowing a small fraction $\eps < 1/2$ of the samples to come from an arbitrary distribution. While early work in the 1960s studied the information-theoretic aspects of one-dimensional inference in this model (including optimal error rates and sample complexity), extensions of these algorithms to higher dimensions required exponential time. It was not until 2016 that the first polynomial-time algorithms were obtained \cite{LRV2016,DKK+19}. This led to a revitalization of the field, with improved algorithms for a variety of problems such as mean estimation \cite{KS17,diakonikolas2018robustly,diakonikolas2022streaming} and linear regression \cite{KKM2018,DKS2019,PJL2024,CAT+2020}.

Apart from the fact that robust statistics considers data corruptions rather than missingness, the differences with the current work are as follows:

\noindent\textbf{(i)} Due to the arbitrary nature of outliers, identifiability is only possible when the corruption fraction satisfies $\eps < 1/2$. If a majority of the samples are corrupted, the dataset may consist of two equally sized subsets corresponding to different underlying distributions, in which case it is impossible to determine which subset contains the inliers. The robust statistics literature has nevertheless considered the regime $\eps > 1/2$. Since identifiability is impossible in this case, the goal shifts to outputting a list of candidate solutions with the guarantee that at least one is close to the ground truth. Algorithms of this type are known as \emph{list-decoding} algorithms \cite{balcan2008discriminative,charikar2017learning}, and such algorithms have been developed for several tasks, including mean estimation \cite{charikar2017learning,DKS18,diakonikolas2022list} and linear regression \cite{karmalkar2019list,raghavendra2020list}.

\noindent\textbf{(ii)} Even in the regime $\eps < 1/2$, where identifiability is possible, consistency—i.e., the property that the estimation error vanishes as the number of samples tends to infinity—is still unattainable in robust statistics. For example, for Gaussian mean estimation with an $\eps$ fraction of arbitrary corruptions, the optimal error is $\Theta(\eps)$, regardless of the sample size (and analogous lower bounds hold for other distributions, with different dependencies on $\eps$). In contrast, \Cref{def:cont_model} allows for consistency due to the additional structure imposed on the missingness pattern.

\paragraph{Truncated Statistics} 
Although it is a form of Missing Not At Random, this large subfield of statistics—tracing back to Galton, Pearson, and Lee \cite{Gal1897, Pea1902, PL1908}—developed largely orthogonally to the rest of the missing data literature. Truncated statistics concerns scenarios in which there is a truncation set (which may be known or unknown to the learning algorithm), and only samples that fall within this set are observed. Despite early work on this problem, efficient algorithms for fundamental tasks were obtained only in the last decade, including Gaussian mean estimation \cite{DGTZ2018,KTZ2019} and linear regression \cite{DGTZ2019,DRZ2020,DSYZ2021}. 

More concretely, the notion of missingness used in truncated statistics is defined as follows (we state it for identity-covariance Gaussians to match the inlier distribution considered in this paper). Samples are drawn from $\mathcal{N}(\mu, I)$ but are revealed only if they fall in some subset $S \subseteq \mathbb{R}^d$ whose probability mass is assumed to be lower bounded by a parameter $\alpha > 0$; otherwise, the samples are hidden (e.g., represented by a special symbol $\perp$). More generally, the literature also considers a setting in which hidden samples are completely unobserved, so the algorithm does not know the ratio of missing to visible samples.

At a high level, this can be viewed as a special case of the $\eps$-realizable contamination model considered in this paper (\Cref{def:cont_model}) with $\eps = 1$, but there are important differences. First, there is no analogue in \Cref{def:cont_model} of the requirement that the truncation set have probability mass bounded away from zero. As a result, the adversary in \Cref{def:cont_model} could make all samples missing, which is the main reason the problem is unsolvable when $\eps = 1$ (the sample complexity in \Cref{eq:sample_compl} diverges as $\eps \to 1$).

Second, even when such a lower-bound condition is imposed, additional nuances arise in the truncated statistics literature depending on whether the truncation set $S$ is known to the algorithm. If the set is unknown, mean estimation remains information-theoretically impossible to arbitrary accuracy \cite{DGTZ2018}. Estimation becomes possible only if (i) the algorithm has oracle access to $S$ \cite{DGTZ2018}, or (ii) the set has bounded complexity, for example bounded VC dimension or bounded Gaussian surface area \cite{KTZ2019}. In the latter case, there is an information–computation trade-off, providing evidence that polynomial-time algorithms often require more samples than the information-theoretic optimum \cite{DKPZ2024}.

\section{Omitted Details from \Cref{sec:prelims}} 

\subsection{Useful Facts} \label{app:useful_facts}

\begin{fact}[Gaussian tail bound]
Let $Z \sim N(0,1)$. Then for all $x > 0$,
\begin{align*}
    \Pr[Z \ge x] \le \frac{1}{x\sqrt{2\pi}} e^{-x^2/2}.
\end{align*}
\end{fact}

\begin{fact}[Mills ratio inequality \cite{gordon1941values}]\label{facr:mills}
    Let $\phi,\Phi$ denote the pdf and cdf of $\cN(0,1)$ respectively. The following holds for all $x>0$:
    \begin{align*}
        x < \frac{\phi(x)}{1-\Phi(x)} < x + \frac{1}{x} .
    \end{align*}
\end{fact}

\begin{fact}[Dvoretzky--Kiefer--Wolfowitz (DKW) inequality]\label{fact:dkw}
Let $X_1,\dots,X_n$ be i.i.d.\ real-valued random variables with cumulative distribution function $F$, and let the empirical cumulative distribution function be $F_n(x) := \frac{1}{n}\sum_{i=1}^n \mathbf{1}\{X_i \le x\}$.
Then, for all $\varepsilon>0$,
\begin{align*}
    \Pr\!\left[ \sup_{x\in\mathbb{R}} \bigl|F_n(x)-F(x)\bigr| > \varepsilon \right]
\le 2\,e^{-2n\varepsilon^2}.
\end{align*}
\end{fact}

\begin{restatable}[Maximal coupling (see, e.g., \cite{Roc2024})]{fact}{FACTMAXIMAL}\label{fact:maximal}
    Let $P$ and $Q$ be distributions on some domain $\cX$. It holds that $\dtv\left( P ,Q \right) = \inf_{\Pi}\Pr_{(X,Y) \sim \Pi}[X \neq Y]$ where the infimum is over all possible couplings between $P$ and $Q$. Moreover, there exists $\Pi$ such that $\Pr_{(X,Y) \sim \Pi}[X \neq Y] = \dtv\left( P ,Q \right)$.
\end{restatable}

\begin{fact}[see, e.g., Corollary 4.2.13 in \cite{Vershynin18}]\label{fact:cover_cor_ver}
    Let $\xi>0$. There exists a set $\cC$ of unit vectors of $\R^d$ such that $|\cC| < (1+2/\xi)^d$ and for every $u \in \R^{d}$ with $\|u\|_2=1$ it holds $\min_{y \in \cC} \|y-u\|_2 \leq \xi$.
\end{fact}

\begin{corollary}[see, e.g., Exercise 4.4.3 (b) in \cite{Vershynin18}]\label{fact:cover}
    There exists a subset $\cC$ of the $d$-dimensional unit ball with $\abs{\cC}\leq 7^{d}$  such that   $\|x\|_2 \le 2\max_{v\in \cC}\abs{v^\top x}$ for all $x\in \mathbb{R}^d$ and  $\|A\|_{\op} \leq 3 \max_{x \in \cC} x^\top A x$ for every symmetric $A \in \R^{d \times d}$.
\end{corollary}

\begin{fact}\label{fact:exp-norm}
Let $x \sim \cN(0, I_d)$ and let $j \ge 1$. Then $\mathbb{E}[\|x\|^j] \leq (2\sqrt{j d})^{\,j}$.
\end{fact}
\begin{proof}
    For a standard normal $z \sim N(0,1)$ it is well known that
$\E|z|^j \le (2\sqrt{j})^{\,j}$.
Since $\|x\|^j = \big(\sum_{i=1}^d x_i^2\big)^{j/2}$ and for $r \ge 1$ we have
$(\sum_{i=1}^d a_i)^r \le d^{r-1} \sum_{i=1}^d a_i^r$,
applying this with $a_i = x_i^2$ and $r = j/2$ gives
\begin{align*}
    \|x\|^j \le d^{j/2 - 1} \sum_{i=1}^d |x_i|^j.
\end{align*}
Taking expectations and using identical marginals,
\begin{align*}
    \E\|x\|^j \le d^{j/2}\, \E|z|^j
\le d^{j/2}\, (2\sqrt{j})^{\,j}
= (2\sqrt{j d})^{\,j}.
\end{align*}
\end{proof}

\begin{proposition}[Le Cam's lemma]\label{lem:LeCam}
For any distributions $P_1$ and $P_2$ on $\mathcal{X}$, we have
\begin{align*}
    \inf_{\Psi} \left\{ \Pr_{X \sim P_1}(\Psi(X) \neq 1) + \Pr_{X \sim P_2}(\Psi(X) \neq 2) \right\} 
= 1 - \DTV(P_1,P_2),
\end{align*}
where the infimum is taken over all tests $\Psi: \mathcal{X} \to \{1,2\}$.
\end{proposition}

\paragraph{Legendre Polynomials} In this work, we make use of the Legendre Polynomials which are orthogonal polynomials over $[-1,1]$. Some of their properties are:

\begin{fact}[\cite{Sze67}]\label{fact:legendre}
The Legendre polynomials $P_k$ for $k\in\Z$, satisfy the following properties:
\begin{enumerate}
    \item $P_k$ is a $k$-degree polynomial and $P_0(x)=1$ and $P_1(x)=x$.
    \item $\int_{-1}^1 P_i(x)P_j(x) \d x=2/(2i+1)\1\{i=j\}$, for all $i,j\in \Z$.
    \item $|P_k(x)|\leq 1$ for all $|x|\leq 1$.
    \item $|P_k(x)|\leq (4|x|)^k$ for $|x| \geq 1$.
    \item $P_k(x)=(-1)^k P_k(-x)$.
    \item $P_k(x)=2^{-k}\sum_{i=1}^{\lceil k/2 \rceil}\binom{k}{i}\binom{2k-2i}{k}x^{k-2i}$.
\end{enumerate}
    
\end{fact}

\subsection{Hermite Analysis}\label{app:hermite}

\begin{definition}[Hermite tensor]\label{def:Hermite-tensor}
For $k\in \N$ and $x\in\R^n$, we define the $k$-th Hermite tensor as
\[
(H_k(x))_{i_1,i_2,\ldots,i_k}=\frac{1}{\sqrt{k!}}\sum_{\substack{\text{Partitions $P$ of $[k]$}\\ \text{into sets of size 1 and 2}}}\bigotimes_{\{a,b\}\in P}(- I_{i_a,i_b})\bigotimes_{\{c\}\in P}x_{i_c}\; .
\]
\end{definition}

\begin{fact} \label{fact:othor-tran}
If $v \in \R^d$ is a unit vector it holds $H_k(v^\top x) = \langle v^{\otimes k}, H_k(x) \rangle$.
\end{fact}

\begin{fact}\label{fact:exp-h-tensor}
    $\E_{x \sim \cN(\mu,I)}[H_k(x)] = \mu^{\otimes k}/\sqrt{k!}$.
\end{fact}

Hermite polynomials form a complete orthogonal basis of the vector space $L_2(\R,\cN(0,1))$ of all functions $f:\R \to \R$ such that $\E_{x\sim \cN(0,1)}[f^2(x)]< \infty$. There are two commonly used types of Hermite polynomials. The \emph{physicist's } Hermite polynomials, denoted by $H_k$ for $k\in \Z$ satisfy the following orthogonality property with respect to the weight function $e^{-x^2}$: for all $k,m \in \Z$, $\int_\R H_k(x) H_m(x) e^{-x^2} \d x = \sqrt{\pi} 2^k k! \mathbf{1}(k=m)$. The \emph{probabilist's} Hermite polynomials $H_{e_k}$ for $k\in \Z$ satisfy $\int_\R H_{e_k}(x) H_{e_m}(x) e^{-x^2/2} \d x = k! \sqrt{2\pi}  \mathbf{1}(k=m)$ and are related to the physicist's polynomials through $H_{e_k}(x)=2^{-k/2}H_k(x/\sqrt{2})$. 
	We will mostly use the \emph{normalized probabilist's} Hermite polynomials $h_k(x) = H_{e_k}(x)/\sqrt{k!}$, $k\in \Z$ for which $\int_\R h_k(x) h_{m}(x) e^{-x^2/2} \d x = \sqrt{2\pi} \mathbf{1}(k=m)$.
	These polynomials are the ones obtained by Gram-Schmidt orthonormalization of the basis $\{1,x,x^2,\ldots\}$ with respect to the inner product $\langle{f},{g}\rangle_{\cN(0,1)}=\E_{x \sim \cN(0,1)}[f(x)g(x)]$. Every function $f \in L_2(\R,\cN(0,1))$ can be uniquely written as $f(x) = \sum_{i \in \Z} a_i h_i(x)$ and we have $\lim_{n \rightarrow \infty}\E_{x \sim \cN(0,1)}[(f(x)- \sum_{i =0}^n a_i h_i(x))^2] = 0$ (see, e.g., \cite{andrews_askey_roy_1999}).
Extending the normalized probabilist's Hermite polynomials to higher dimensions, an orthonormal basis of $L_2(\R^d,\cN(  0,  I_d))$ (with respect to the inner product 
$\langle f, g\rangle = \E_{x \sim \cN(0,   I_d)}[f(x)g(x)]$)  
can be formed by all the products of one-dimensional Hermite polynomials, i.e., $h_{a}(x) = \prod_{i=1}^d h_{a_i}(x_i)$, 
for all multi-indices $a \in \Z^d$ (we are now slightly overloading notation by using multi-indices as subscripts). The total degree of $h_ a$ is $|a|=\sum_{i=1}^d a_i$.

\begin{claim}[Univariate Gaussian shift bound]\label{fact:univariate-shift}
Let $p:\mathbb R\to\mathbb R$ be a polynomial of degree at most $d$ satisfying $\E_{x \sim \cN(0,1)}[p(x)^2]=1$.
Then for every $\mu\in\mathbb R$,
$\E_{x \sim \cN(0,1)}\left[p(x+\mu)^2\right]\le e^{d\mu^2}$.
\end{claim}

\begin{proof}
Let $h_k(x)=He_k(x)/\sqrt{k!}$ denote the normalized probabilists' Hermite polynomials.
Expand $p$ in the orthonormal Hermite basis:
\[
p(x)=\sum_{k=0}^d c_k\,h_k(x),
\qquad \sum_{k=0}^d c_k^2=\mathbb E_{x \sim \cN(0,1)}[p(x)^2]=1.
\]

\indent The following identity holds for the normalized probabilists' Hermite polynomials:
\[
h_k(x+\mu)
=\sum_{r=0}^k \binom{k}{r}\mu^{\,k-r}\sqrt{\frac{r!}{k!}}\,h_r(x).
\]

Using orthonormality of $\{h_r\}$ and independence of the coefficients in the expansion, we have that
\[
\E_{x \sim \cN(0,1)}[h_k(x+\mu)^2]
=\sum_{r=0}^k \binom{k}{r}^2 \mu^{2(k-r)}\,\frac{r!}{k!}
= \sum_{s=0}^k \binom{k}{s}\frac{\mu^{2s}}{s!},
\]
where $s=k-r$.
Bounding $\binom{k}{s}\le k^s/s!$ yields
\[
\binom{k}{s}\frac{\mu^{2s}}{s!}
\le \frac{(k\mu^2)^s}{s!}.
\]
Therefore
\[
\E_{x \sim \cN(0,1)}[h_k(x+\mu)^2]
\le \sum_{s=0}^\infty \frac{(k\mu^2)^s}{s!}
= e^{k\mu^2}.
\]

\noindent Now expand $p(x+\mu)$:
\[
 \E_{x \sim \cN(0,1)}[p(x+\mu)^2]
= \sum_{k=0}^d c_k^2\,\E_{x \sim \cN(0,1)}[h_k(x+\mu)^2]
\le \sum_{k=0}^d c_k^2 e^{k\mu^2}
\le e^{d\mu^2}\sum_{k=0}^d c_k^2
= e^{d\mu^2}.
\]
\end{proof}

\begin{claim}[Multivariate Gaussian shift bound]\label{fact:multivariate-shift}
Let $p:\mathbb R^n\to\mathbb R$ be a polynomial of total degree at most $D$ satisfying $\E_{x \sim \cN(0,1)}[p(x)^2]=1$.
For a multi-index $\alpha\in\mathbb N^n$, define $h_\alpha(x)=\prod_{i=1}^n h_{\alpha_i}(x_i)$ and $|\alpha|=\sum_{i=1}^n \alpha_i$.
Then for every $\mu\in\mathbb R^n$, it holds $\E_{x \sim \cN(0,1)}[p(x+\mu)^2]\le e^{D\|\mu\|_2^2}$.
\end{claim}

\begin{proof}
Expand $p$ in the multivariate orthonormal Hermite basis:
\[
p(x)=\sum_{|\alpha|\le D} c_\alpha\,h_\alpha(x),
\qquad \sum_{|\alpha|\le D} c_\alpha^2=\E_{x \sim \cN(0,1)}[p(x)^2]=1.
\]
\noindent Because $h_\alpha(x)=\prod_i h_{\alpha_i}(x_i)$, and the coordinates of $x$ are independent, we have
\[
\E_{x \sim \cN(0,1)}[h_\alpha(x+\mu)^2]
= \prod_{i=1}^n \E_{x \sim \cN(0,1)}[h_{\alpha_i}(x_i+\mu_i)^2].
\]

\noindent Applying the univariate bound of \Cref{fact:univariate-shift} to each coordinate,
\[
\E_{x \sim \cN(0,1)}[h_{\alpha_i}(x_i+\mu_i)^2]\le e^{\alpha_i\mu_i^2},
\]
so
\[
\E_{x \sim \cN(0,1)}[h_\alpha(x+\mu)^2]
\le \exp\!\left(\sum_{i=1}^n \alpha_i\mu_i^2\right)
\le \exp\!\left(|\alpha|\,\|\mu\|_2^2\right)
\le e^{D\|\mu\|_2^2}.
\]

\noindent Finally,
\[
\mathbb E[p(X+\mu)^2]
=\sum_{|\alpha|\le D} c_\alpha^2\,\mathbb E[h_\alpha(X+\mu)^2]
\le \sum_{\alpha} c_\alpha^2\, e^{D\|\mu\|_2^2}
= e^{D\|\mu\|_2^2}.
\]
\end{proof}

\begin{claim}\label{cl:bound_hermite_norm}
    Let $H_k$ denote the $k$-th Hermite tensor for $d$ dimensions. Then, the following bound holds: $\|H_k(x)\|_2 \leq d^{k/2}(1+\|x\|^k)2^{O(k)}$.
\end{claim}

\begin{proof}

For a degree-$k$ tensor $A$, we use $A^\pi$ to denote the matrix that $A^{\pi}_{i_1,\ldots,i_k} = A_{\pi(i_1,\ldots,i_k)}$. Note that $\|A\|_2 = \|A^{\pi}\|_2$. Then from the definition of Hermite tensor, we have that
\begin{align*}
    H_k(x) =  \frac{1}{\sqrt{k!}} \sum_{t=1}^{\lfloor k/2 \rfloor} \sum_{\text{Permutation $\pi$ of $[k]$}} \frac{1}{2^t t! (k-2t)!} \left( I^{\otimes t}  x^{\otimes (k-2t)}\right)^{\pi}.
\end{align*}

Thus the norm is

\begin{align*}
    \left\| H_k(x) \right\|_2
    &= \left\| \frac{1}{\sqrt{k!}} \sum_{t=1}^{\lfloor k/2 \rfloor} \sum_{\text{Permutation $\pi$ of $[k]$}} \frac{1}{2^t t! (k-2t)!} \left( I^{\otimes t}  x^{\otimes (k-2t)}\right)^{\pi}   \right\|_2 \\
    &\leq \sum_{t=1}^{\lfloor k/2 \rfloor}   \frac{\sqrt{k!}}{2^t t! (k-2t)!} \max\left( \|I^{\otimes t}\|_2 \|x\|_2^{k-2t},1  \right)\\
    &\leq \sum_{t=1}^{\lfloor k/2 \rfloor}   \frac{\sqrt{k!}}{2^t t! (k-2t)!} \max(d^{t/2}\|x\|_2^{k-2t},1) \\
    &\leq \sum_{t=1}^{\lfloor k/2 \rfloor}   \frac{\sqrt{k!}}{2^t t! (k-2t)!} \left( d^{t/2}\|x\|_2^{k-2t} + 1 \right) \\
    &\leq \sum_{t=1}^{\lfloor k/2 \rfloor}   \frac{\sqrt{k!}}{2^t t! (k-2t)!} \left( d^{t/2}\max(\|x\|_2^{k},1) + 1 \right)\\
    &\leq  2d^{k/2}(1+\|x\|^k)\sum_{t=1}^{\lfloor k/2 \rfloor}   \frac{\sqrt{k!}}{2^t t! (k-2t)!}.
\end{align*}
One can see that the denominator is minimized when $t = k/2-O(\sqrt{k})$. Using that, we have that the right hand side above is at most $d^{k/2}(1+\|x\|^k)2^{O(k)}$.
    
\end{proof}

We restate and prove the following concentration of Hermite moments.
\HERMITECONCENTRATION*
\begin{proof}
    Consider one entry $\hat T_{i_1i_2\cdots i_k}$ of the estimator. It holds
    \begin{align}
        \Var(\hat T_{i_1 i_2\cdots i_k}) 
        &= \frac{1}{n}\Var_{x \sim P'}(H_k(x)_{i_1 i_2\cdots i_k}) \notag \\
        &\leq \frac{1}{n} \E_{x \sim P'}[(H_k(x)_{i_1 i_2\cdots i_k})^2] \notag \\
        &\leq \frac{1}{n(1-\eps)} \E_{x \sim \cN(\mu,I)}\left[(H_k(x)_{i_1 i_2\cdots i_k})^2 \right] \tag{by \Cref{def:cont_model}}\\
        &\leq \frac{1}{n(1-\eps)} \E_{x \sim \cN(\mu,I)}\left[\|H_k(x)\|_2^2\right] \notag\\
        &\leq \frac{d^{k/2}2^{O(k)}}{n(1-\eps)} \left( 1+ \E_{x \sim \cN(\mu,I)}\left[ \|x\|_2^k \right] \right) \tag{by \Cref{cl:bound_hermite_norm}}\\
        &\leq \frac{d^{k}2^{O(k)}k^{k/2}}{n(1-\eps)}, \label{eq:var-bound}
    \end{align}
    where the last step can be shown as follows:
    \begin{align*}
        \E_{x \sim \cN(\mu,I)}\left[ \|x\|_2^k \right]
        &= \E_{z \sim \cN(0,I)}\left[ \|z + \mu\|_2^k \right]\\
        &\leq \E_{z \sim \cN(0,I)}\left[ (\|z\|_2 + \|\mu\|_2)^k \right]\\
        &\leq 2^{k-1}\left( \E_{z \sim \cN(0,I)}\left[ \|z\|_2^k \right] + \|\mu\|_2^k \right)\\
        &\leq 2^{O(k)}\left( (kd)^{k/2} + 2^{O(k) }\log(1/(1-\eps))^{k/2} \right) \tag{using \Cref{fact:exp-norm}}\\
        &\leq 2^{O(k)}k^{k/2}d^{k/2}\log(1/(1-\eps))^{k/2}.
    \end{align*}

    Having the variance bound of inequality \eqref{eq:var-bound}, an application of Chebyshev's inequality yields that if the number of samples is $n > C \frac{d^{2k}2^{O(k)}k^{k/2}\log(1/(1-\eps))^{k/2}}{(1-\eps)\eta^2 \tau'}$ then
    \begin{align*}
        \Pr\left[ | \hat T_{i_1 i_2\cdots i_k} -  T_{i_1 i_2\cdots i_k} | > \frac{\eta}{d^{k/2}}  \right] \leq \frac{d^{2k}2^{O(k)}k^{k/2}\log(1/(1-\eps))^{k/2}}{n(1-\eps)\eta^2} \leq \tau'.
    \end{align*}
     We will use $\tau' = \tau d^{-k}$. By union bound the probability of having $| \hat T_{i_1 i_2\cdots i_k} -  T_{i_1 i_2\cdots i_k} | \leq \frac{\eta}{d^{k/2}}$ for all entries simultaneously is at least $1-\tau$. In that event we have that $\| \hat T - T\|_2 = \sqrt{\sum_{i_1\cdots i_k}| \hat T_{i_1 i_2\cdots i_k} -  T_{i_1 i_2\cdots i_k} |^2} \leq \eta$ which completes the proof.
\end{proof}

\section{Sample Complexity Lower Bound}\label{app:sample_lb}

We restate and prove the following result.

\begin{restatable}[Sample complexity lower bound]{theorem}{SAMPLELB} \label{thm:sample_compl}
    For every $\eps \in (0,1), \delta  \in (0,\log^{-1/2}(1+\tfrac{\eps}{1-\eps}))$ and $ n\in \Z_+$ the following holds.
    If $\cA$ is an algorithm that uses samples from an $\eps$-corrupted version of a Gaussian $\cN(\mu,1)$ and outputs $\hat{\mu}$ such that $\|\mu - \hat{\mu}\|_2 \leq \delta$ with probability at least $0.9$ then the sample complexity of $\cA$ is
    \begin{equation}\label{eq:lowerbound}
        n \geq \frac{1}{1-\eps}  \exp\left( \Omega\left( \frac{\log\left(1+\tfrac{\eps}{1-\eps}\right)}{\delta} \right)^2  \right).
    \end{equation}
\end{restatable}

\begin{remark}
Some remarks follow:
    \begin{itemize}
        \item The bound in \Cref{eq:lowerbound} agrees with the sample complexity upper and lower bound shown in \cite{MVB+2024} (Theorems 5 and 6 therein): Our model coincides with theirs when $\sigma = 1$ and $q = 1$. Solving for the second term (which is the dominant term) in Theorem~5 or 6 to be equal to $\delta^2$ yields (up to constant factors) the same expression as the right-hand side of \Cref{eq:lowerbound}.
        \item (Small $\eps$ regime) When $\eps \to 0$ the bound becomes $\exp(\Omega(\eps/\delta)^2)$.
        \item (Large $\eps$ regime) When $\eps \to 1$ the bound behaves like $\exp\left(\Omega(\tfrac{1}{\delta}\log(\tfrac{1}{1-\eps}))^2\right)$.
    \end{itemize}
\end{remark}

The argument for showing the theorem consists of showing that there exist two distributions in this contamination model that are close in total variation distance.

\begin{lemma}\label{lem:coupling}
    For every $\eps \in (0,1), \delta>0, n\in \Z_+$ the following holds. Consider the two Gaussians $P_1 = \cN(-\delta/2,1)$ and $P_2 = \cN(\delta/2,1)$. There exist distributions $Q_1,Q_2$ on $\R$ such that
    \begin{itemize}
        \item $Q_1$ is an $\eps$-corrupted version of $P_1$ according to \Cref{def:cont_model} and $Q_2$ is $\eps$-corrupted version of $P_2$.
        \item $\DTV(Q_1^{\otimes n},Q_2^{\otimes n}) \leq \frac{n}{1-\eps}  e^{-\Omega\left( \log(1+\eps/(1-\eps))/\delta \right)^2}$.
    \end{itemize}
\end{lemma}

\noindent We first show how \Cref{thm:sample_compl} follows given the lemma

\begin{proof}[Proof of \Cref{thm:sample_compl}]
    Define the following hypothesis testing problem: With probability $1/2$ all samples come from $Q_1$ and with probability $1/2$ all samples come from $Q_2$. If a mean estimator existed that had accuracy $\delta/2$ with probability $0.9$ then we would be able to solve the testing problem with probability $0.9$. However by Le Cam's lemma (\Cref{lem:LeCam}), every testing algorithm has probability of failure at least $\frac{1}{2}\left( 1 -  \DTV(Q_1^{\otimes n},Q_2^{\otimes n})\right)$. In order for that probability of failure to be less than $0.1$ we need $n > \frac{1}{1-\eps}  e^{\Omega\left( \log(1+\eps/(1-\eps))/\delta \right)^2}$.
\end{proof}

\begin{proof}[Proof of \Cref{lem:coupling}]
    Fix the threshold $t := \frac{\log\left(1+\tfrac{\eps}{1-\eps}\right)}{\delta}$ through this proof. Let $p_1(x),p_2(x)$ denote the pdfs of $P_1,P_2$.
    We define the functions $q_1,q_2$ as shown below:
    \begin{align*}
    q_1(x) &=
    \begin{cases}
        p_1(x)         & x \in (0,\infty) \\
        p_2(x)         & x \in [-t,0] \\
        (1-\eps)p_1(x) & x \in (-\infty,-t)
    \end{cases}
    &\qquad
    q_2(x) &=
    \begin{cases}
        (1-\eps)p_2(x)         & x \in (t,\infty) \\
        p_1(x)         & x \in [0,t] \\
        p_2(x) & x \in (-\infty,0)
    \end{cases}
    \end{align*}

    \begin{claim}
        It holds $(1-\eps)p_1(x) \leq q_1(x) \leq p_1(x)$ and $(1-\eps)p_2(x) \leq q_2(x) \leq p_2(x)$ for all $x \in \R$.
    \end{claim}
    \begin{proof}
        We do the check for the first part of the claim involving $p_1$ and $q_1$. The check for the second part is identical.
        The only non-trivial part of check is showing that $p_2(x) \geq (1-\eps)p_1(x)$ for all $x \in [-t,0]$.

        Recall that $p_2$ is the pdf of $\cN(0,\delta/2)$ and $p_1$ is the pdf of $\cN(-\delta/2,1)$. We thus want to solve for $(1-\eps) p_1(x) \leq p_2(x)$. Plugging in the pdf of the two Gaussians
        \begin{align*}
            \exp\left( -\frac{(x-\delta/2)^2}{2} + \frac{(x+\delta/2)^2}{2} \right) \geq 1-\eps
        \end{align*}
        Solving the above yields $x \geq -\log(1 + \tfrac{\eps}{1-\eps})/\delta$. 
        Therefore, $p_2(x) \geq (1-\eps)p_1(x)$ for all $x \in [-t,0]$.
    \end{proof}

    Since $(1-\eps)p_1(x) \leq q_1(x) \leq p_1(x)$ the function $q_1$ induces a definition of an $\eps$-corruption of $P_1$ according to \Cref{def:cont_model}. That is, a sample from $Q_1$ is generated according to the following procedure: 
     With probability $\int_{\R} q_1(x) \d x$ the sample is drawn $q_1(x)/\int_{\R} q_1(x) \d x$, and with probability $1-\int_{\R} q_1(x) \d x$ the sample is set to the special symbol $\perp$.

    Similarly, $(1-\eps)p_2(x) \leq q_2(x) \leq p_2(x)$ and $q_2$ induces an $\eps$-corrupted version of $P_1$, that we denote by $Q_2$. Samples from $Q_2$ are generated as follows: with probability $\int_{\R} q_2(x) \d x$ the sample is drawn from $q_2(x)/\int_{\R} q_2(x) \d x$, and with probability $1-\int_{\R} q_2(x) \d x$ the sample is set to the special symbol $\perp$.

    By symmetry of our setup, the probability of the sample not being deleted (set to $\perp$) is the same $\int_{\R} q_1(x) \d x = \int_{\R} q_2(x) \d x$. Denote by $\alpha$ this probability. Also denote by $\tilde{Q}_1$ the conditional distribution of $Q_1$ conditioned on the sample not being $\perp$ and let $\tilde{Q}_2$ denote the corresponding conditional distribution for $Q_2$.

    In the following we will show that $\DTV( Q_1^{\otimes n},  Q_2^{\otimes n}) \leq \frac{n}{1-\eps}  e^{-\Omega\left( \log(1+\eps/(1-\eps))/\delta \right)^2}$. By \Cref{fact:maximal} it suffices to find a coupling $\Pi$ between the two joint distributions $Q_1^{\otimes n}, Q_2^{\otimes n}$ with probability of disagreement at most  $\frac{n}{1-\eps}  e^{-\Omega\left( \log(1+\eps/(1-\eps))/\delta \right)^2}$. That is, we need to define a joint distribution $\Pi$ on two sets of $n$ samples $((X_1,\ldots,X_n),(Y_1,\ldots,Y_n))$ such that (i)  the marginals are $(X_1,\ldots,X_n) \sim \tilde Q_1^{\otimes n}$ and $(Y_1,\ldots,Y_n) \sim \tilde Q_2^{\otimes n}$ respectively (i.e., it is a valid coupling) and (ii) the probability of disagreement is $\Pr_{\Pi}[(X_1,\ldots,X_n) \neq (Y_1,\ldots,Y_n)] \leq \frac{n}{1-\eps}  e^{-\Omega\left( \log(1+\eps/(1-\eps))/\delta \right)^2}$.

    We define the coupling by defining the data generation process for $((X_1,\ldots,X_n),(Y_1,\ldots,Y_n))$ bellow. In the construction bellow, we will assume that we already have a coupling $\Pi_0$ for the conditional distributions of single samples, i.e., a distribution $\Pi_0$ such that if $(X,Y) \sim \Pi_0$ it holds $X\sim \tilde Q_1, Y \sim \tilde Q_2$ (i.e., $\Pi_0$ is a coupling between $\tilde Q_1$ and $\tilde Q_2$)  and $\Pr_{(X,Y) \sim \Pi_0}[X \neq Y] \leq \frac{n}{1-\eps}  e^{-\Omega\left( \log(1+\eps/(1-\eps))/\delta \right)^2}$. We will show why $P_0$ exists at the end; for now we will conclude the construction of $\Pi$ using $\Pi_0$. We define the sample generation process for $\Pi$ as follows:
    \begin{enumerate}
        \item Draw $c_i \sim \Bernoulli(\alpha)$ for $i \in [n]$ (recall that $\alpha$ is the probability of ``missing'' samples).
        \item For each $i \in [n]$:
        \begin{enumerate}
            \item If $c_i = 1$, then draw $(X_i,Y_i) \sim \Pi_0$.
            \item Else, set $(X_i,Y_i)=(\perp,\perp)$.
        \end{enumerate}
    \end{enumerate}

    Note that in the above construction, each $X_i$ is distributed according to $Q_1$ and each $Y_i$ follows $Q_2$ thus the above defines a valid coupling between $Q_1^{\otimes n},Q_2^{\otimes n}$. For the probability of disagreement we have the following:

    \begin{align*}
        \Pr[(X_1,\ldots,X_n) \neq (Y_1,\ldots,Y_n)] 
        &\leq \sum_{i=1}^n \Pr[X_i \neq Y_i] \\
        &= \sum_{i=1}^n \Pr[X_i \neq Y_i | c_i = 1] \Pr[c_i=1] + \Pr[X_i \neq Y_i | c_i = 0] \Pr[c_i=0]\\
        &\leq \sum_{i=1}^n \Pr[X_i \neq Y_i | c_i = 1] \\
        &= \sum_{i=1}^n \Pr_{(X_i,Y_i) \sim \Pi_0}[X_i \neq Y_i] \\
        &\leq  \frac{n}{1-\eps}  \exp\left( -\Omega\left( \frac{\log\left(1+\tfrac{\eps}{1-\eps}\right)}{\delta} \right)^2  \right)\;.
    \end{align*}

    It suffices to show that the coupling $\Pi_0$ with $\Pr_{(X,Y) \sim \Pi_0}[X \neq Y] \leq  \frac{n}{1-\eps}  e^{-\Omega\left( \log(1+\eps/(1-\eps))/\delta \right)^2}$ exists. We show this by bounding the TV distance between the conditional distributions $\tilde{Q}_1,\tilde{Q}_2$ and defining $\Pi_0$ to be the maximal coupling (cf. \Cref{fact:maximal}).
    \begin{align}
        \DTV(\tilde{Q}_1,\tilde{Q}_2) 
        &= \frac{1}{2} \int_{-\infty}^{+\infty} \left| \frac{q_1(x)}{\int_{\R} q_1(x) \d x} - \frac{q_2(x)}{\int_{\R} q_2(x) \d x} \right| \\
        &= \frac{1}{2\alpha} \int_{-\infty}^{\infty} \left| q_1(x) - q_2(x) \right| \d x\\
        &= \frac{1}{2\alpha} \left(  \int_{-\infty}^{-t} \eps p_1(x)  \d x  + \int_{t}^{+\infty} \eps p_2(x)  \d x  \right) \tag{by definition of $q_1,q_2$}\\
        &= \frac{1}{2\alpha} \left(  \int_{-\infty}^{-t} \eps \phi(x + \delta/2)  \d x  + \int_{t}^{+\infty} \eps \phi(x + \delta/2)  \d x  \right) \tag{$\phi$ is the pdf of $\cN(0,1)$}\\
        &\leq \frac{1}{2\alpha} 2 e^{-\Omega\left(\frac{1}{\delta}\log(1+\tfrac{\eps}{1-\eps}) - \delta/2\right)^2} \tag{see below}\\
        &\leq \frac{1}{1-\eps}  e^{-\Omega\left(\frac{1}{\delta}\log(1+\tfrac{\eps}{1-\eps})\right)^2} . \tag{see below} 
    \end{align}
    Where the second to last line uses the standard Gaussian tail bound $\Pr_{z \sim \cN(0,1)}[z > r] \leq e^{-r^2/2}$ for every $r \geq 0$. We are using this with $r:= \frac{1}{\delta}\log(1+\tfrac{\eps}{1-\eps}) - \delta/2$. Note that this is non-negative because $\eps \in (0,1)$ and $\delta^2 \leq \log(1+\eps/(1-\eps))$. The final line uses $\alpha:=1-\eps$ and $\eps \in (0,1)$ and $\delta^2 \leq \log(1+\eps/(1-\eps))$.

\end{proof}

\section{Sample Complexity Upper Bound}\label{app:sample_ub}

We restate the main result for the sample complexity of one dimensional estimation below. We will then use this together with a cover argument to show a multi-variate estimator in \Cref{thm:brute_force}.

\begin{restatable}[Sample complexity upper bound]{theorem}{SAMPLEUB}\label{prop:one_d_estimator}
There exists a computationally efficient algorithm such that the following holds for any $\eps \in (0,1)$, $\delta \in (0,\log^{1/2}(1+\tfrac{\eps}{1-\eps}))$ and $\tau \in (0,1)$.
    The algorithm takes as input $\eps,\delta,\tau$, draws $n=\exp\left( O\left( \frac{\log(1+\tfrac{\eps}{1-\eps})}{\delta}  \right)^2 \right) \frac{\log(1/\tau)}{\eps^2(1-\eps)}$ samples from an $\eps$-corrupted version of $\cN(\mu,1)$ under the contamination model of \Cref{def:cont_model}, and it returns $\hat{\mu}$ such that it holds $|\hat{\mu} - \mu| \leq \delta$ with probability at least $1-\tau$.
\end{restatable}
 
\begin{remark}
    Some remarks follow:
    \begin{itemize}
        \item If $\eps$ is not known to the algorithm it can be easily estimated by taking the fraction of samples that are equal to $\perp$.
        \item The algorithm's sample complexity matches the lower bound of \Cref{thm:sample_compl} and the sample complexity upper bound of \cite{MVB+2024} up to the $\eps^{-2}$ factor.
    \end{itemize}
\end{remark}

We start with the claim that if the cdfs of two Gaussians are  multiplicatively close to each other, then the means of the Gaussians must also be appropriately close. The following structural lemma quantifies this.

\begin{lemma}\label{lem:structural_original}
Let $\xi,t,\eps$ be real numbers with $\xi/2>0$, $t<-\xi/2$ and $\eps \in (0,1)$.
    Let two Gaussians $P_{+}=\cN(\xi/2,1)$ and $P_{-}=\cN(-\xi/2,1)$ and denote by $F_{+}(x)$ and $F_{-}(x)$ their cumulative distribution functions (cdfs).
    If $t$ is a point for which $F_{+}(t) \geq (1-\eps)F_{-}(t)$, then $\xi \leq \frac{\log\left( 1 + \tfrac{\eps}{1-\eps} \right)}{|t|}$.
\end{lemma}

\begin{proof}

    It suffices to prove the claim for the extreme case, i.e., that $F_{+}(t) = (1-\eps)F_{-}(t)$ implies $\xi \leq \frac{\log\left( 1 + \tfrac{\eps}{1-\eps} \right)}{|t|}$.
    Let $\Phi(x)$ denote the cdf of $\cN(0,1)$. Then the equation $F_{+}(t) = (1-\eps)F_{-}(x)$ is equivalent to $\Phi(t-\xi/2)/\Phi(t+\xi/2) = 1-\eps$. Taking logarithms on both sides and doing some further rewriting, we have
    \begin{align}
        \log(1-\eps)
        &= \log\left( \frac{\Phi(t-\xi/2)}{\Phi(t+\xi/2)}  \right) 
        = \log( \Phi(t-\xi/2))- \log (\Phi(t+\xi/2)) \notag \\
        &=- \int_{t-\xi/2}^{t+\xi/2} \frac{\d}{\d y}\log \Phi(y) \d y
        = - \int_{t-\xi/2}^{t+\xi/2} \frac{\phi(y)}{\Phi(y)} \d y , \label{eq:integral}
    \end{align}
    where $\phi(y)$ denotes the pdf of $\cN(0,1)$.  
    Recall the Mills ratio inequality (cf. \Cref{facr:mills}):
    \begin{align*}
        x \leq \frac{\phi(x)}{1-\Phi(x)} \leq x + \frac{1}{x} \quad\quad\text{ $\forall x>0$}.  
    \end{align*}
    However, due to our assumption $t<-\xi/2$, the variable $y$ inside the integral in \Cref{eq:integral} is always negative.
    We can obtain a version of Mill's ratio inequality for negative reals by using the symmetry properties $\phi(x)=\phi(-x)$, $1-\Phi(x)=\Phi(-x)$:
    \begin{align*}
        -y \leq \frac{\phi(y)}{\Phi(y)} \leq -y - \frac{1}{y} \quad\quad\text{ $\forall y<0$}.
    \end{align*}
    Combining the left part of the above inequality with \Cref{eq:integral}, we obtain
    \begin{align*}
        \int_{t-\xi/2}^{t+\xi/2} y \d y \geq  - \int_{t-\xi/2}^{t+\xi/2} \frac{\phi(y)}{\Phi(y)} \d y = \log(1-\eps) .
    \end{align*}
    Using $\int_{t-\xi/2}^{t+\xi/2} y \d y = \frac{1}{2}((t+\xi/2)^2-(t-\xi/2)^2)=2t\xi/2=-2|t| \xi/2$ and rearranging the above inequality, we finally obtain
    \begin{align*}
        \xi \leq \frac{-\log(1-\eps)}{|t|} = \frac{\log\left( 1 + \tfrac{\eps}{1-\eps} \right)}{|t|} .
    \end{align*}
\end{proof}

We will use the contrapositive and shifted version of \Cref{lem:structural_original} that is stated below. This is contrapositive because it is saying that large difference in the mean of two Gaussians translates to large multiplicative gap of their cdfs, and it is shifted because it includes an arbitrary shift $\mu$ in the means of both Gaussians.

\begin{corollary}\label{lem:structural}
    Let $\mu,\xi,t,\eps$ be reals with, $\xi > 0$, $t + \xi/2 < 0$ and $\eps \in (0,1)$.
    If $P_1=\cN(\mu,1)$ and $P_2=\cN(\mu + \xi,1)$ are two Gaussians with $\xi > \frac{\log\left( 1 + \tfrac{\eps}{1-\eps} \right)}{|t|} $, then their cdfs $F_1,F_2$ satisfy $F_{2}(t+\mu+\xi/2) < (1-\eps)F_{1}(t+\mu+\xi/2)$.
\end{corollary}

\begin{proof}
    
    First, with $\frac{n}{1-\eps}$ samples one can learn an approximation $\hat{F}$ to the cumulative distribution function $F$ of the corrupted distribution of the non-missing samples (\Cref{fact:dkw}). That is, with probability at least $1-2e^{-2n \eta^2}$ we have
    \begin{align}
        \left| \hat{F}(x) - F(x)  \right| \leq \eta \label{eq:approx_guarantee}
    \end{align}
    for all $x \in \R$.
    For the remainder of the proof fix $t:=\frac{\log\left( 1 + \tfrac{\eps}{1-\eps} \right)}{\delta}$.
    We will use $\eta := \eps \Phi(-t)$ (where $\Phi$ denotes the cdf of $\cN(0,1)$) and we will set $n$ a sufficiently large multiple of $\eta^{-2} \log(1/\tau)$ so that the probability of failure is at most $\tau$.

    Let $u := \hat{F}^{-1}(\Phi(-t))$, i.e, the point for which it holds $\hat{F}(u) = \Phi(-t)$.
    Consider the Gaussian distribution $\cN(\mu_0,1)$ where $\mu_0 := u + t$. This is exactly the Gaussian whose cdf $F_0$ satisfies $F_0(u)=\Phi(-t)$.
    The algorithm then is this: We simply return $\hat{\mu} = \mu_0 = u + t$. 

    Given \Cref{lem:structural} it is easy to see why this has accuracy $O(\delta)$:
    Let $F^{*}$ be the cdf of the ground truth Gaussian (inlier distribution). By the fact that $\hat{F}(u) = \Phi(-t)$, \Cref{eq:approx_guarantee} and the definition of our contamination model, we have that $(1-O(\eps))\Phi(-t) \leq F^{*}(u) \leq (1+O(\eps))\Phi(-t)$.
    By \Cref{lem:structural}, if $\tilde F$ is the cdf of a unit-variance Gaussian with mean larger than $\mu_0 + \frac{C\log\left( 1+\tfrac{\eps}{1-\eps} \right)}{u-\mu_0-\delta/2}$, then it would hold $\tilde F(t) < (1-C\eps) \Phi(-t)$. Thus, this Gaussian could not be the ground truth one as its cdf evaluated on point $t$ falls ouside of the interval $[(1-O(\eps))\Phi(-t) , \leq (1+O(\eps))\Phi(-t)]$. Similarly we can rule out any Gaussian with mean smaller than $\mu_0 -\frac{C\log\left( 1 + \tfrac{\eps}{1-\eps} \right)}{|u-\mu_0-\delta/2|}$ from being the ground truth. This means that the point $\mu_0$ is within $\frac{2C\log\left( 1 + \tfrac{\eps}{1-\eps} \right)}{|u-\mu_0-\delta/2|}=\frac{2C\log\left( 1 + \tfrac{\eps}{1-\eps} \right)}{t+\delta/2}=O(\delta)$ from the ground truth Gaussian's mean, where in the last step we used that $t = \log\left( 1 + \tfrac{\eps}{1-\eps} \right)/\delta$ and $\delta^2 \leq \log(1+\tfrac{\eps}{1-\eps})$. By adjusting the constants we can turn this $O(\delta)$ into just $\delta$.
\end{proof}

We now show the extension of the algorithm to multiple dimensions.

\begin{restatable}[Multivariate estimator]{theorem}{BRUTEFORCE}\label{thm:brute_force}
    Let $d \in \Z_+$ denote the dimension, and $C$ be a sufficiently large absolute constant.
    Let $\eps \in (0,1)$, be a contamination parameter and $\delta \in (0,\log^{1/2}\left( 1 + \tfrac{\eps}{1-\eps} \right))$ be an accuracy parameter.
    Let  $\mu \in \R^d$ be an (unknown) vector.
    There exists an algorithm that takes as input $\eps$, $\delta$, draws $n = \exp\left( O\left( \frac{\log(1+\tfrac{\eps}{1-\eps})}{\delta}  \right)^2 \right) \frac{d+\log(1/\tau)}{\eps^2(1-\eps)}$ points from an $\eps$-corrupted version of  $\cN(\mu,  I_d)$ under the contamination model of \Cref{def:cont_model}, outputs a $\widehat{\mu}$ such that $\|\widehat{\mu}-\mu \|_2  \leq \delta$ with probability at least $1-\tau$. Moreover, it runs in time $2^{O(d)}\poly(n, d)$.
\end{restatable}

\begin{proof}[Proof of \Cref{thm:brute_force}]
Denote by $T=\{x_i\}_{i=1}^n, x_i\in \mathbb{R}^d$ the points from the $\eps$-corrupted version of $\cN(\mu,  I)$ and denote by $\cC$ the the cover set of \Cref{fact:cover}.
    The algorithm is the following: First, using the algorithm from \Cref{prop:one_d_estimator},  calculate a $m_v$ for each $v \in \cC$  such that $\abs{m_v-v^\top \mu}\le \delta/8$ (see next paragraph for more details on this step). Then, output the solution of the following linear program (note that the program always has a solution, as it is satisfied by $\widehat{\mu}=\mu$):
    \begin{align*}
        &\text{ Find }\widehat{\mu}\in \mathbb{R}^d \text{ s.t.}\\
        &\abs{v^\top\widehat{\mu}-m_v}\le \delta/4, \forall v\in \mathcal{C} \;.
    \end{align*}
    The claim is that this solution $\widehat{\mu}$ is indeed close to the target $\mu$, since 
   \begin{align}
        \|\mu-\widehat{\mu}\|_2 &\leq 2\max_{v\in \mathcal{C}}\abs{v^\top(\mu-\widehat{\mu})}  \tag{using \Cref{fact:cover}}\\
        &\le 2\max_{v\in \mathcal{C}}(\abs{v^\top \mu -m_v}+\abs{m_v -v^\top\widehat{\mu} }) \notag\\
        &\le 2(\eps/8 + \eps/4) < \eps \;. \label{eq:temp4234}
    \end{align}
    We now explain how to obtain the approximations $m_v$ with the guarantee $\abs{m_v-v^\top \mu}\le \delta/8$.
    Fixing a direction $v \in \mathcal{C}$, we note that  $v^\top x\sim \cN(v^\top \mu,1)$ thus $\{v^\top x_i\}_{i=1}^m$ is a set of samples from an $\eps$-corrupted version of $ \cN(v^\top \mu,1)$. Thus, if we apply algorithm  from \Cref{prop:one_d_estimator} with probability of failure $\tau' = \tau/|\cC|$, the event $\abs{m_v-v^\top \mu}\le \delta/8$ will hold with probability at least $1-\tau/|\mathcal{C}|$. By union bound, the probability all the events for $v \in \cC$  hold simultaneously is at least $1-\tau$. The number of samples for this application of \Cref{prop:one_d_estimator} is $2^{O(\log(1+\eps/(1-\eps)/\delta)^2} \frac{\log(1/\tau')}{\eps^2(1-\eps)} = 2^{O(\log(1+\eps/(1-\eps)/\delta)^2} \frac{\log(|\cC|/\tau)}{\eps^2(1-\eps)} = 2^{O(\log(1+\eps/(1-\eps)/\delta)^2} \frac{d+ \log(1/\tau)}{\eps^2(1-\eps)}$. 
    
    We conclude with the runtime analysis.
    The runtime to find the $m_v$'s is $O(|\cC|\poly(nd)) = 2^{O(d)}\poly(nd)$ since for each fixed $v \in \cC$ we need $\poly(nd)$ time to calculate the projection $\{x_i^\top v\}$ of our dataset onto  $v$ and $\poly(n)$ time to run the one-dimensional estimator.
    The linear program can be solved using the ellipsoid algorithm. Consider the separation oracle that exhaustively checks all $2^{O(d)}$ constraints. We need $\poly(d)\log(\frac{R}{r})$ calls to that separation oracle, where $R,r$ are the radii of the bounding spheres of the feasible region. First, $R\leq \eps$, because we have already shown in \eqref{eq:temp4234}  that the feasible set belongs in a ball of radius $\eps$ around $\mu$. Regarding the upper bound $r$, note that all $\widehat{\mu}$ inside a ball of radius $\delta/8$ around $\mu$ are feasible since $\abs{v^\top\widehat{\mu}-m_v}\le \abs{v^\top\widehat{\mu}-v^\top \mu}+\abs{v^\top\mu-m_v}\le \|\widehat{\mu}-\mu\|+ \delta/8 \leq \eps/4$. This means that $r=\eps/4$. Hence the total runtime for solving the LP is $2^{O(d)} \poly(d)$ or simply $2^{O(d)}$.

\end{proof}

\section{Hardness in Other Restricted Models of Computation}\label{app:different_models_hardnesss}

We give a brief summary of known information-computation gaps for Non-Gaussian Component Analysis problem (\Cref{prob:generic_hypothesis_testing}) 
for different restricted models of computation. Since we have already shown that our problem is an instance of NGCA, we obtain immediately corollaries for Low-Degree Polynomials (\Cref{cor:LDP-hardness}) and PTFs (\Cref{cor:PTFs}). The result for SoS has a few mild but cumbersome to verify conditions, so we refer the readers directly to \cite{diakonikolas2024sum} for the formal statements.

\subsection{Hardness in the Low-Degree Polynomial Class of Algorithms}\label{sec:prior-work-ldlr}

 We start with the Low-Degree Polynomial (LDP) model, which we describe in more detail. We will consider tests that are thresholded polynomials of low-degree, i.e., output $H_1$ if the value of the polynomial exceeds a threshold and $H_0$ otherwise. We need the following notation and definitions.
For a distribution $D$ over $\cX$, we use $D^{\otimes n}$ to denote the joint distribution of $n$ i.i.d.\ samples from $D$.
For two functions $f:\cX \to \R$, $g: \cX \to R$ and a distribution $D$, we use $\langle f, g\rangle_{D}$ to denote the inner product $\E_{X \sim D}[f(X)g(X)]$.
We use $ \|f\|_{D}$ to denote $\sqrt{\langle f, f \rangle_{D} }$.
We say that a polynomial $f(x_1,\dots,x_n):\R^{n \times d} \to \R$ has sample-wise degree $(r,\ell )$ if each monomial uses at most $\ell$ different samples from $x_1,\dots,x_n$ and uses degree at most $r$ for each of them.
Let $\cC_{r,\ell}$ be linear space of all polynomials of sample-wise degree $(r,\ell)$ with respect to the inner product defined above.
For a function $f:\R^{n \times d} \to \R$, we use $f^{\leq r, \ell}$ to be the orthogonal projection onto $\cC_{r,\ell}$   with respect to the inner product $\langle \cdot , \cdot \rangle_{D_0^{\otimes n}}$.  Finally, for the null distribution $D_0$ and a distribution $P$, define the likelihood ratio $\overline{P}^{\otimes n}(x) := {P^{\otimes n}(x)}/{D_0^{\otimes n}(x)}$.

\begin{definition}[$n$-sample $\tau$-distinguisher]\label{def:distinguisher}
For the hypothesis testing problem between  $D_0$ (null distribution) and $D_1$ (alternate distribution) over $\cX$,
we say that a function $p : \cX^n \to \R$ is an $n$-sample $\tau$-distinguisher if $|\E_{X \sim  D_0^{\otimes n}}[p(X)] - \E_{X \sim D_1^{\otimes n}}[p(X)]| \geq \tau \sqrt{\Var_{X \sim D_0^{\otimes n}} [p(X)] }$. We call $\tau$ the \emph{advantage} of the polynomial $p$. 
\end{definition}
Note that if a function $p$ has advantage $\tau$, then the Chebyshev's inequality implies that one can furnish a test $p':\cX^n \to \{D_0,D_1\}$ by thresholding $p$ such that the probability of error under the null distribution is at most $O(1/\tau^2)$. 
We will think of the advantage $\tau$ as the proxy for the inverse of the probability of error and we will show that the advantage of all polynomials up to a certain degree is $O(1)$. 
It can be shown that for hypothesis testing problems of the form of   \Cref{prob:generic_hypothesis_testing}, 
 the best possible advantage among all polynomials in $\cC_{r,\ell}$ is captured by the low-degree likelihood ratio (see, e.g.,~\cite{BBH+2021,kunisky2022notes}):
\begin{align*}
    \left\| \E_{v \sim \cU(S)}\left[ \left( \overline{P}_{A,v}^{\otimes n}  \right)^{\leq r, \ell } \right]  - 1  \right\|_{D_0^{\otimes n}},
\end{align*}
where in our case $D_0 = \cN( 0,I)$.

It has been known by \cite{BBH+2021} that a lower bound for the SQ dimension translates to an upper bound for the low-degree likelihood ratio. Given this, one can obtain the corollary regarding the hardness of NGCA:

        \begin{theorem}[Information-computation gap for NGCA in LDP]\label{cor:low-deg-hardness-general-problem} 
        Let $c$ be a sufficiently small positive constant and consider the hypothesis testing problem of  \Cref{prob:generic_hypothesis_testing}  the distribution $A$ matches the first $t$  moments with $\cN(0,I)$. For any $d \in \Z_+$ with $d = t^{\Omega(1/c)}$, any $n \leq \Omega(d)^{(t+1)/10}/\chi^2(A,\normal( 0, I))$ and any even integer $\ell < d^{c}$, we have that
        \begin{align*}
            \left\| \E_{v \sim \cU(S)}\left[ \left( \overline{P}_{A,v}^{\otimes n}  \right)^{\leq \infty, \ell } \right]  - 1  \right\|_{D_0^{\otimes n}} \leq 1\;.
        \end{align*}
    \end{theorem}

The interpretation of this result is that unless the number of samples used $n$ is greater than $\Omega(d)^{(t+1)/10}/\chi^2(A,\normal( 0, I))$, any polynomial of degree roughly up to $d^{c}$  fails to be a good test (note that any polynomial of degree $\ell$ has sample-wise degree at most $(\ell,\ell)$).

We now show the corollary for the robust mean estimation problem of this paper. Recall that the hypothesis testing problem of \Cref{thm:SQ} includes as a special case the NGCA problem with $A$ being the $\eps$-corrupted version of $N(\delta v, I)$ where $v$ is a unit vector and the corruption adversary from \Cref{def:cont_model} uses as $f(x)$ the function from \Cref{lem:correct_moments}. For this distribution $A$ we have that (i) it matches the first $\Omega(\gamma^2/\log \gamma)$ moments with $\cN(0,1)$, where $\gamma := \frac{1}{\delta}\log(1+\tfrac{\eps/2}{1-\eps/2})$ and (ii) $\chi^2(A,\cN(0,1) = O(\frac{1}{1-\eps})$; this part is not included in \Cref{lem:correct_moments} because it was not needed for the SQ lower bound, but it immediately follows by using that $f(x)/\int f(x) \d x \leq \frac{g(x) + p(x) \1(|x| \leq 1)}{1-\eps} \leq \tfrac{\phi(x-\delta)+\eps}{1-\eps}$ (in the proofs of \Cref{lem:match_inside_interval,lem:correct_moments} it can be seen that $g(x)$ bounded by a translated Gaussian and that the polynomial $p(x)$ has small absolute value in $[-1,1]$).

\begin{corollary}[Hardness of mean estimation against Low-Degree Polynomials]\label{cor:LDP-hardness}
    Consider the same hypothesis testing problem as in \Cref{thm:SQ} and let $m$ be defined as in \Cref{thm:SQ}. There is a way for the adversary of \Cref{def:cont_model} to corrupt the alternative hypothesis $P =\cN(\delta u,I)$ into a distribution $\tilde P$ such that the following holds. Let $c$ be a sufficiently small positive constant. For any $d \in \Z_+$ with $d = m^{\Omega(1/c)}$, any $n \leq \Omega(d)^{(m+1)/10}(1-\eps)$ and any even integer $\ell < d^{c}$, we have that
        \begin{align*}
            \left\| \E_{v \sim \cU(S)}\left[ \left( \tilde{P}^{\otimes n}  \right)^{\leq \infty, \ell } \right]  - 1  \right\|_{D_0^{\otimes n}} \leq 1\;.
        \end{align*}
\end{corollary}

This is interpreted as a tradeoff between $d^{\Omega(m)}(1-\eps)$ and super-polynomial runtime.

\subsection{Hardness against PTF Tests}
\label{sec:PTFs}

In the LDP class of the previous section, the goodness of a test is quantified by the advantage, defined in \Cref{def:distinguisher}. Hardness of a problem in this class are shown by ruling out existence of polynomials with small advantage. That definition is based on the idea that one can obtain a test by thresholding the polynomial in the midpoint of the expectations for the two distributions. Thus rulling out existence of polynomials with small advantage rulls out the construction of such tests. However it still leaves open the possibility of some other kind of thresholded polynomial might still succeed. As such, a more natural class of tests is the one consisting of all possible thresholded polynomials (i.e., with arbitrary thresholdings).

The information-computation gap for NGCA in this class is the one below.

\begin{theorem}[Information-computation gap for NGCA against PTF Tests (\cite{diakonikolas2025ptf})]
\label{thm:main}
There exists a sufficiently large absolute constant $C^*$ such that the following holds.
For any $c^* \in (0, 1/4)$, $d,k,n,m \in \mathbb Z_+$ such that (i) $m$ is even, (ii) $\max(k,m) < d^{c^*/C^*}$, and (iii) $n < d^{ (1/4 - c^*) m}$,
we have that 
if $p: \R^{n \times d} \mapsto \R$ is a degree-$k$ polynomial,  and $A$ is a distribution on $\R$ that matches the first $m$ moments with $\normal(0,1)$, then:
\begin{align}\label{eq:thm_conclusion}
\Bigg| \E_{  \substack{\vec v \sim \cU(S)\\ \multix \sim P_{A,v} }}
\left[  \sgn(p(\vec x^{(1)}, \cdots, \vec x^{(n)})) \right]
-  
\E_{ \multix \sim \normal(\vec 0, \vec I)}    
\left[  \sgn(p(\vec x^{(1)}, \cdots, \vec x^{(n)})) \right] \Bigg| 
\leq 0.11.
\end{align}
where $P_{A,v}$ denotes the hidden direction distribution from \Cref{def:hidden_dir_dist}, and 
$\sgn : \R \to \{0,1\}$ is the sign function with $\sgn(x) = 1$ if and only if $x \geq 0$.
\end{theorem}

The corollary for robust mean estimation is stated below.

\begin{corollary}[Hardness of mean estimation against PTFs]
\label{cor:PTFs}
Consider the same hypothesis testing problem as in \Cref{thm:SQ} and let $m$ be defined as in \Cref{thm:SQ}. There is a way for the adversary of \Cref{def:cont_model} to corrupt the alternative hypothesis $P =\cN(\delta u,I)$ into a distribution $\tilde P$ such that the following holds.
There exists a sufficiently large absolute constant $C^*$ such that the following holds.
For any $d,k,n,m \in \mathbb Z_+$ such that (i) $m$ is even, (ii) $\max(k,m) < d^{c^*/C^*}$, and (iii) $n < d^{ \Omega(m)}$,
we have that 
if $p: \R^{n \times d} \mapsto \R$ is a degree-$k$ polynomial,  then:
\begin{align}\label{eq:thm_conclusion}
\Bigg| \E_{  \substack{\vec v \sim \cU(S)\\ \multix \sim P_{A,v} }}
\left[  \sgn(p(\vec x^{(1)}, \cdots, \vec x^{(n)})) \right]
-  
\E_{ \multix \sim \normal(\vec 0, \vec I)}    
\left[  \sgn(p(\vec x^{(1)}, \cdots, \vec x^{(n)})) \right] \Bigg| 
\leq 0.11.
\end{align}
where $P_{A,v}$ denotes the hidden direction distribution from \Cref{def:hidden_dir_dist}, and 
$\sgn : \R \to \{0,1\}$ is the sign function with $\sgn(x) = 1$ if and only if $x \geq 0$.
\end{corollary}

For an arbitrary polynomial $p$, 
the runtime for this computation 
is on the order of $\poly((nd)^k)$ and thus, \Cref{thm:main} implies an inherent trade-off between the 
exponential runtime $(nd)^{d^{\Omega(1)}}$, and 
the sample complexity $d^{\Omega(m)}$ for the family of PTF tests.

\end{document}